%% file: assembler_arxiv.tex
\documentclass{article}

\usepackage{microtype}
\usepackage{graphicx}
\usepackage{subfigure}
\usepackage{booktabs}
\usepackage{enumitem}
\usepackage{bbm}

\usepackage{hyperref}
\usepackage{amsmath}
\usepackage{amssymb}
\usepackage{amsthm}
\usepackage{authblk}

\usepackage{natbib}
\setcitestyle{authoryear,round,citesep={;},aysep={,},yysep={;}}

\usepackage[margin=1in]{geometry}

\usepackage[capitalize,noabbrev]{cleveref}

\theoremstyle{plain}
\newtheorem{theorem}{Theorem}[section]
\newtheorem{proposition}[theorem]{Proposition}
\newtheorem{lemma}[theorem]{Lemma}
\newtheorem{corollary}[theorem]{Corollary}
\theoremstyle{definition}
\newtheorem{definition}[theorem]{Definition}

\theoremstyle{remark}
\newtheorem{remark}[theorem]{Remark}
\input{macros.tex}

\def\max{{\rm max}}
\def\whp{{\rm w.h.p.}}

\author[1,2]{Yatin Dandi}
\author[2]{Emanuele Troiani}
\author[1]{Luca Arnaboldi}
\author[1]{Luca Pesce}
\author[2]{\\ Lenka Zdeborová}
\author[1]{Florent Krzakala}

\affil[1]{Information Learning and Physics Laboratory, \'Ecole Polytechnique F\'ed\'erale de Lausanne (EPFL)}
\affil[2]{Statistical Physics Of Computation Laboratory, \'Ecole Polytechnique F\'ed\'erale de Lausanne (EPFL)}

\title{The Benefits of  Reusing Batches for Gradient Descent in Two-Layer Networks: Breaking the Curse of Information and Leap Exponents}

\begin{document}
\date{}
\maketitle

\begin{abstract}
We investigate the training dynamics of two-layer neural networks when learning multi-index target functions.
We focus on multi-pass gradient descent (GD) that reuses the batches multiple times and show that it significantly changes the conclusion about which functions are learnable compared to single-pass gradient descent. 
In particular, multi-pass GD with finite stepsize is found to overcome the limitations of gradient flow and single-pass GD given by the information exponent \citep{BenArous2021} and leap exponent \citep{abbe2023sgd} of the target function. We show that upon re-using batches, the network achieves in just {\it two} time steps an overlap with the target subspace even for functions not satisfying the staircase property \citep{abbe2021staircase}.
We characterize the (broad) class of functions efficiently learned in finite time.
The proof of our results is based on the analysis of the Dynamical Mean-Field Theory (DMFT). 
We further provide a closed-form description of the dynamical process of the low-dimensional projections of the weights, and numerical experiments illustrating the theory.
\end{abstract}
\input{sections/intro.tex}
\input{sections/setting_motivation.tex}
\input{sections/main_results.tex}

\input{sections/asymptotics}
\input{sections/conclusions}

\section{Acknowledgements}
We thank Cedric Gerbelot, Bruno Loureiro and Ludovic Stephan for insightful discussions. We also acknowledge funding from the Swiss National Science Foundation grant SNFS OperaGOST  (grant number $200390$), and SMArtNet (grant number $212049$).

\bibliographystyle{abbrvnat}
\bibliography{biblio}

\newpage

\appendix
\onecolumn
\input{sections/appendix/proofs.tex}
\input{sections/appendix/details_numerics.tex}

\end{document}

%% file: macros.tex
\newcommand{\Ea}[1]{\E\left[#1\right]}
\newcommand{\Eb}[2]{\E_{#1}\left[#2\right]}

\renewcommand{\vec}[1]{\boldsymbol{#1}}

  \providecommand{\R}{\mathbb{R}} 
  \providecommand{\N}{\mathbb{N}} 

  \makeatletter
  \def\sign{\@ifnextchar*{\@sgnargscaled}{\@ifnextchar[{\sgnargscaleas}{\@ifnextchar{\bgroup}{\@sgnarg}{\sgn} }}}
  \def\@sgnarg#1{\sgn\rbr{#1}}
  \def\@sgnargscaled#1{\sgn\rbr*{#1}}
  \def\@sgnargscaleas[#1]#2{\sgn\rbr[#1]{#2}}
  \makeatother

  \providecommand{\0}{\bm{0}}

\providecommand{\E}{\mathbb{E}}%

  \providecommand{\cO}{\mathcal{O}}


\newcommand{\speedup}[1]{{\color{gray}(\ifdim #1 pt > 0.3pt #1\else $< #1$\fi{}$\times$)}}

\usepackage{amsfonts}
\usepackage{amsmath}
\usepackage{amsthm}
\usepackage{amssymb}
\usepackage{dsfont}

\usepackage{xcolor}
\usepackage{color}
\usepackage{graphicx}

\usepackage{verbatim}
\usepackage{xspace} %

\usepackage{enumerate}
\usepackage{enumitem}

\makeatletter
\newsavebox{\@brx}
\newcommand{\llangle}[1][]{\savebox{\@brx}{\(\m@th{#1\langle}\)}%
  \mathopen{\copy\@brx\mkern2mu\kern-0.9\wd\@brx\usebox{\@brx}}}
\newcommand{\rrangle}[1][]{\savebox{\@brx}{\(\m@th{#1\rangle}\)}%
  \mathclose{\copy\@brx\mkern2mu\kern-0.9\wd\@brx\usebox{\@brx}}}
\makeatother

\providecommand{\abs}[1]{\left\lvert#1\right\rvert}
\providecommand{\norm}[1]{\left\lVert#1\right\rVert}

  \providecommand{\R}{\mathbb{R}} %
  \providecommand{\N}{\mathbb{N}} %

  \providecommand{\Eb}[1]{{\mathbb E}\left[#1\right] }

  \providecommand{\0}{\mathbf{0}}

  \providecommand{\cO}{\mathcal{O}}

\renewcommand{\vec}{\mathbf}
  \usepackage{bm}

\providecommand{\mycomment}[3]{\todo[caption={},size=footnotesize,color=#1!20]{\textbf{#2: }#3}}%
\providecommand{\inlinecomment}[3]{%
  {\color{#1}#2: #3}}%
\newcommand\commenter[2]%
{%
  \expandafter\newcommand\csname i#1\endcsname[1]{\inlinecomment{#2}{#1}{##1}}
  \expandafter\newcommand\csname #1\endcsname[1]{\mycomment{#2}{#1}{##1}}
}

  \definecolor{mydarkblue}{rgb}{0,0.08,0.45}
  \usepackage{hyperref}
  \hypersetup{ %
    pdftitle={},
    pdfauthor={},
    pdfsubject={},
    pdfkeywords={},
    pdfborder=0 0 0,
    pdfpagemode=UseNone,
    colorlinks=true,
    linkcolor=mydarkblue,
    citecolor=mydarkblue,
    filecolor=mydarkblue,
    urlcolor=mydarkblue,
    pdfview=FitH}

\usepackage{url}

%% file: sections/intro.tex
\section{Introduction}
\label{sec:main:introduction}

Recent years have witnessed significant theoretical advancements in understanding the dynamics of training neural networks using gradient descent, to unravel the learning mechanisms of these networks, particularly how they adapt to data and identify pivotal features for predicting the target function. Significant progress has been made over the last few years in the case of two-layer networks, in large part thanks to the so-called mean-field analysis \citep{mei2018mean,chizat2018global,rotskoff2018trainability,sirignano2020mean}). Most of the theoretical efforts, in particular, focused either on {\it one-pass} optimization algorithms, where {\it each iteration involves a new fresh batch of data}, or to the limit of gradient flow in the population loss. For high-dimensional synthetic Gaussian data, and a low dimensional target function (a multi-index model), the class of functions efficiently learned by these \textit{one-pass} methods has been thoroughly analyzed in a series of recent works, and have been shown to be limited by the so-called information exponent \citep{BenArous2021} and leap exponent \citep{abbe2022merged, abbe2023sgd} of the target. These analyses have sparked many follow-up theoretical works over the last few months, see, e.g. \citep{damian2022neural,damian_2023_smoothing,dandi2023twolayer,bietti2023learning,ba2023learning, moniri2023theory, mousavihosseini2023gradientbased,zweig2023symmetric}.

However, a common practice in machine learning involves repeatedly traversing the same mini-batch of data. This paper, therefore, aims to go beyond the constraints of single-pass algorithms and to evaluate whether multiple-pass training overcomes these inherent flaws of single-pass methods. We focus on gradient descent, certainly the most straightforward procedure in this family. The theoretical framework we use to prove our main results is based on Dynamical Mean Field Theory (DMFT), which was developed in the statistical physics community \citep{Sompolinsky_88} to analyze correlated systems, and recently made rigorous in the context of high dimensional machine-learning problems in \citep{celentano2021highdimensional, gerbelot2023rigorous}.

Our findings significantly alter the prevailing narrative in the literature. We demonstrate that gradient descent {\it surpasses} the limitations imposed by the information and leap exponents, achieving a positive correlation with the target function for a much broader class than the staircase functions \citep{abbe2021staircase}, even with minimal (that is, two) repetition of data batches. We characterize the (broad) class of functions efficiently learned in finite time. Among the exceptions are {\it symmetric} functions that remain a challenge due to their extended symmetry-breaking times (a natural feature of physical dynamics \citep{bouchaud1998out}). 

With independent Gaussian datapoints as inputs, both one-pass SGD and gradient flow on population loss result in pre-activations remaining distributed as Gaussian random variables. This Gaussianity underlies the analysis of such settings,  starting from the seminal work of \cite{saad.solla_1995_line}.
In contrast, upon re-using batches, the pre-activations develop non-Gaussian components correlated with the targets.  This non-Gaussianity is a crucial aspect of the stark contrast in the learning of directions compared to the one-pass setting. While we establish our results for discrete steps with extensive batches of size $\cO(d)$ where $d$ denotes the input dimension, we expect our conclusions about the learning of new directions to also hold while performing gradient flow on the empirical loss, since such a setup would lead to the development of similar non-Gaussian components in the pre-activations, in contrast to gradient flow on the population loss where the pre-activations remain Gaussian.

Our results demonstrate that contrary to the common wisdom ``the more data the better", gradient descent on the same batch can surpass one-pass SGD on different batches, even when one-pass SGD utilizes a larger number of datapoints. More generally, we believe that our analysis provides insights into incremental learning of features in the presence of correlations between datapoints across batches, which is typical in most high dimensional datasets having a small number of ``latent" factors. Our conclusions follow from a rigorous mathematical proof rooted in DMFT, which we also use to provide an analytic description of the dynamic processes of low-dimensional weight projections. This analysis has interest on its own.
\looseness=-1

%% file: sections/setting_motivation.tex
\section{Setting and main contributions}
\label{sec:main:setting}
 Let $\mathcal{D}$ be the set of data $\{\vec{z}_{\nu} \in {\mathbb R}^d, y_{\nu}\}_{\nu \in [n]}$. The input data $\vec{z}_{\nu}$ are taken as a standard i.i.d. Gaussian, while the labels are generated by a teacher, or target, function $y_\nu = f^\star(\vec z_\nu)$. We consider multi-index target function, dependent on a low-dimensional subspace of the input space:
 \looseness=-1 
\begin{align}
   y_\nu = f^\star(\vec z_\nu) = g^\star\left(\frac{W^\star \vec z_\nu}{\sqrt{d}}\right),\, \vec z_\nu \sim {\cal N}(0,{\mathbbm 1}_d)\,,
   \label{eq:main:teacher_def}
\end{align}  We assume for convenience that $W^\star = \{\vec w^\star_r\}_{r \in [k]} \in \mathbb{R}^{k \times d}$ is normalized row-wise on the sphere $\mathcal{S}_{d-1}(\sqrt{d})$, with orthogonal weights i.e $\langle \vec w^\star_l,\vec w^\star_m \rangle = 0$ for $l \neq m \in [r]$. 

The data are handled to a {\it two-layer} network (the student)~$f$ with first layer weights $W = \{\vec w_i \}_{i \in [p]} \in \mathbb{R}^{p \times d}$ ($p$ is the number of neurons in the hidden layer) and second layer weights $\vec a \in \mathbb{R}^p$
with an activation function $\sigma$, that is:
\begin{align}
    f\left( \frac{W\vec z} {\sqrt{d}}\right) = \sum_{j=1}^p a_j \sigma{\left(\frac{\langle \vec w_j, \vec z \rangle}{ \sqrt{d}}\right)}
\end{align}

Our main goal is to analyze the dynamics of gradient descent that minimizes the empirical Mean Squared Error (MSE) loss $\mathcal{L}$ at time $t \in [T]$:
\begin{align}
  {\cal R}_{\rm emprical} &=\sum_{\nu=1}^n\mathcal{L}\left(\frac{W^{(t)}\vec z_\nu} {\sqrt{d}},f^\star(\vec z_\nu)\right) \\
  &= \frac{1}{2} \sum_{\nu=1}^{n}\left[g^\star\left(\frac{W^\star \vec z_\nu}{\sqrt{d}}\right) - f\left( \frac{W^{(t)}\vec z_\nu} {\sqrt{d}}\right)\right]^2\,, \nonumber 
\end{align}
in the high-dimensional limit where $n,d \!\to\! \infty$ with $n/d \!=\! \alpha \!=\! \Theta(1)$. We use a common assumption that is amenable to rigorous theoretical guarantees: we keep the second layer weights $\vec a$ fixed at initialization. 
For convenience, we further impose the constraint of symmetric initialization common in such analyses \citep{dandi2023twolayer,damian2022neural}. Concretely, we assume that the number of neurons $p$ is even and the weights satisfy at initialization:
\begin{equation}\label{eq:init_sym}
    a_i = -a_{p-i+1}, \quad \vec{w}_i^0 = \vec{w}_{p-i+1}^0 \quad \text{for all } i \in [p/2],
\end{equation}
which ensures that the output $f\left( \frac{W^{(t)}\vec z_\nu} {\sqrt{d}}\right)$ equals $0$ at initialization.
For $i, \in [p/2]$, the weights are initialized as $a_i \sim \frac{1}{p}\mathcal{N}(0,1), \vec w_i^{(0)}\!\!\sim\!\! \mathcal{N}(0,{\mathbbm 1}_d)$
Subsequently, with $\vec{a}$ fixed, the first layer weights $W=\{\vec w_i\}_{i \in [p]}$ are learned using gradient descent, producing the following sequence of iterates up to a final time $T$:\looseness=-1
\begin{align} 
    \label{eq:GD_update}
    \vec w_i^{(\!t+1\!)}\!=\!(1\!-\!\eta\lambda)\vec w^{(\!t\!)}_i \!-\! \eta \!\sum_{\nu=1}^n\!\nabla_{\!\!\vec w^{(\!t\!)}_i}
    \mathcal{L}\!\left(\!\frac{W^{(t)}\vec z_\nu} {\sqrt{d}},f^\star(\vec z_\nu)\!\!\right)
\end{align}
where $\eta\!\in\!\mathbb{R}$ is the learning rate and $\lambda \in \mathbb{R}$ is the explicit regularisation. We may refer to these steps as the \textit{representation learning} steps, in which the first layer weights learn how to adapt to the low dimensional structure identified by the teacher subspace $W^\star$.  

Our \textbf{main contributions} in this paper are the following: 
\vspace{-0.3cm}
\begin{itemize}[noitemsep,leftmargin=1em,wide=0pt]
    \item We characterize the class of multi-index targets that can be learned efficiently by two-layer networks trained with a finite number of iterations of gradient descent in the high dimensional limit $(d \to \infty)$ with large batch sizes ($n = \alpha d, \alpha = O(1)$). We establish a strong separation between what can be learned with one-pass algorithms (that use new fresh batches at every step) and multi-pass gradient approaches that can 
    use the same batch many times (see Figs.     \ref{fig:main:single_index} and \ref{fig:main:multiple_index} for examples).
    \item We show that while both gradient flow \citep{bietti2023learning} and single-pass
    algorithms suffer from the curse of the information exponent \citep{BenArous2021}, and are limited to staircase learning \citep{abbe2023sgd}, requiring a diverging number of iterations for non-staircase functions, some of these problems become {\it trivial} when allowing reusing samples multiple times, and features can be learned in just $T=2$ iterations. This disproves, in particular, a recent conjecture by \citep{abbe2023sgd}.
    \item The simplest examples of directions that cannot be learned in a finite number of steps relate to symmetries in the target function. This includes phase retrieval \citep{maillard2020phase} or the specialization transition in committees, as discussed in the Bayes optimal approaches of single-index \citep{barbier2019optimal} and multi-index \citep{Aubin_2019} models.  
    \item The proof of our results is based on the concept of ``hidden progress", and crucially uses the rigorous Dynamical Mean Field Theory (DMFT) \citep{celentano2021highdimensional,gerbelot2023rigorous}. This has an interest on its own as it provides a sharp example of how DMFT can help to understand batch reusing to go beyond the current state-of-the-art results.
    \item Finally, we use DMFT to provide a closed-form description of the dynamics of gradient descent for two-layer nets. Keeping track of the correlations induced by re-using the same batch leads to a set of integro-differential equations.  We provide rigorous theoretical guarantees in the correlated samples regime without assuming the resampling of a fresh new batch for each iteration of the algorithm. We corroborate the theoretical claims with numerical simulations (See  \href{https://github.com/IdePHICS/benefit-reusing-batch}{\href{https://github.com/IdePHICS/benefit-reusing-batch}{https://github.com/IdePHICS/benefit-reusing-batch}.
}).  
\end{itemize}
\vspace{-0.5cm}

\paragraph{Other Related works --} A major issue in machine learning theory is figuring out how well two-layer neural networks adapt to low-dimensional structures in the data. Different results have tightly characterized the limitations of networks in which the first layer of weights $W$ is kept fixed, i.e. equivalent to kernel approaches \citep{Dietrich1999,Ghorbani2019, Ghorbani2020, bordelon20a,loureiro_learning_2021,Cui2021}. This class of learning algorithms, although amenable to theoretical analysis, is unable to learn features in the data. Therefore, one central avenue of research in this context is to understand the efficiency of the \textit{representation learning} (or feature learning) when training with gradient-based algorithms to overcome the limitations of the kernel regime. Sharp separation results between the performance of neural networks at initialization (random features) and trained with only one step of gradient descent (with a large learning rate) have been offered \citep{ba2022high, damian2022neural, dandi2023twolayer}. 

The class of features efficiently learned with multiple steps of one-pass SGD with one sample per batch is characterized by the \textit{information exponent} ($\rm{IE}$) \citep{BenArous2021} of the target function. In the context of single-index learning, denoting $\ell$  the $\rm{IE}$ of the target, the algorithm needs $T\!=\! O(d^{\ell -1})$  steps to perform weak recovery of the teacher direction, i.e., obtaining an overlap between learned weights and $\vec w^\star$ better than random guessing \citep{BenArous2021}. Recently, these results have been improved up to the Correlational Statistical Query (CSQ) lower bound of $d^{{\max(\frac{\ell}{2}},1)}$, by considering an appropriate smoothing of the loss \citep{damian_2023_smoothing}. A generalization to large batch one-pass SGD is in \citep{dandi2023twolayer}. 

Similarly, multi-index feature learning presents an unavoidable computational barrier for one-pass algorithms. \citep{abbe2021staircase} first characterizes a hierarchical picture of learning in the Boolean data case: informally, the features efficiently learned at each step of the one-pass algorithm need to be linearly connected with the previously learned features. This concept is formalized by the definition of the staircase property \citep{abbe2021staircase}. This hierarchical picture of learning is extended to large batches in the SGD and non-Boolean data in \citep{abbe2022merged,abbe2023sgd, dandi2023twolayer}. Moreover, \citep{abbe2023sgd} conjecture that re-using the batch can reduce the sample complexity of the target with leap $\ell$ only up to $O(d^{\max(\frac{\ell}{2},1)})$, corresponding to the lower bound for Correlational Statistical Query (CSQ) algorithms.

We disprove this conjecture and show that the sample complexity for a large class of functions can be reduced to $O(d)$ {\it independently} of the leap exponent $\ell$. More generally, our results show that CSQ lower bounds and the notions of staircase property and information exponent are limited to online-SGD on Gaussian/Boolean data,  and do not describe the class of functions inherently easy or hard to learn by gradient-based methods. We also show that learning non-even single-index functions does not require techniques such as spectral warm-start \citep{chen2020learning}.\looseness=-1

Dynamical Mean Field Theory has a long history in statistical physics. Early theories of dynamics in complex systems were pioneered in soft spin glass models \citep{SZipp81} and toy models of random feature deep networks \citep{Sompolinsky_88}. The DMFT approach used in this paper was first proposed as a way to study ``hard spins" in spin glass models \citep{Opper_92, Opper_94}, and was later generalized to ``soft spins" \citep{Cu02} and more realistic models in condensed matter \citep{GKKR96}. In the context of learning, DMFT was used for optimization problems \citep{mannelli2019afraid, SKUZ19, SBCKUZ20, mannelli2021just} and for analyzing the behavior and the noise of gradient-based algorithms \citep{mignacco2020dynamical, mignacco2021stochasticity, Mignacco_2022}.  From the mathematics point of view, these DMFT equations were first proven rigorously in the seminal work of \citep{arous1997symmetric} in the context of spin glasses. Important progress was achieved recently with rigorous proofs of the DMFT equations for multi-index models \citep{celentano2021highdimensional, gerbelot2023rigorous} that we use to prove our main results.\looseness=-1




%% file: sections/main_results.tex
\begin{figure*}[t]
\includegraphics[width=0.95\textwidth]{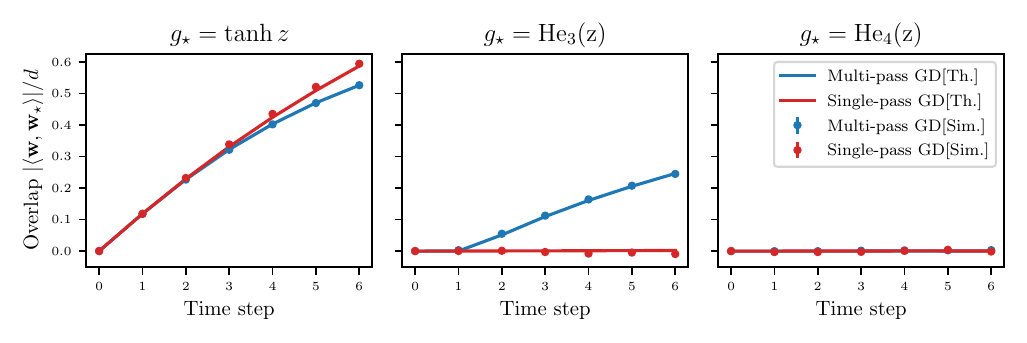}
\vspace{-2em}
    \caption{\textbf{One-pass and multi-pass GD for single-index models --} The {\it overlap} $\left|\frac{\langle \vec w^\star,  \hat{\vec{w}}\rangle }{d}\right|$ between the learned weight and the target/teacher direction, is plotted as a function of the iteration time of both single-pass (red) and multi-pass (blue) GD.  Continuous lines are given theory,  dots are simulations. 
    \textbf{Left: Easy finite-T learnable single-index target} $g^\star \!=\! \tanh$: both one-pass and multi-pass GD obtain positive correlation after a finite number of iterations as the information exponent of the target is $\ell\!=\!1$.  \textbf{Center: Multi-pass finite-T learnable single-index target}: $g^\star \!=\! \mathrm{He}_3$. Multi-pass GD achieves a non-zero correlation in just two steps, but the one-pass algorithm learns nothing. 
    \textbf{Right: Finite-time nonlearnable single-index targets} $ g^\star \!=\! \mathrm{He}_4$; the target function is even and thus, as stated in Thm.~\ref{thm:main:2step_learning}, breaking this symmetry is hard in finite number of steps, resulting in a vanishing correlation with the teacher direction $\vec w^\star$ for both algorithms in any finite time. (Simulation are averaged over $32$ runs,     $d=5000$, with $\sigma = \rm relu$, $n=3d$, $p=1$, $\eta = 0.1$).
    }
    \label{fig:main:single_index}
\end{figure*}

\begin{figure*}[t]
\includegraphics[width=1.05\textwidth]{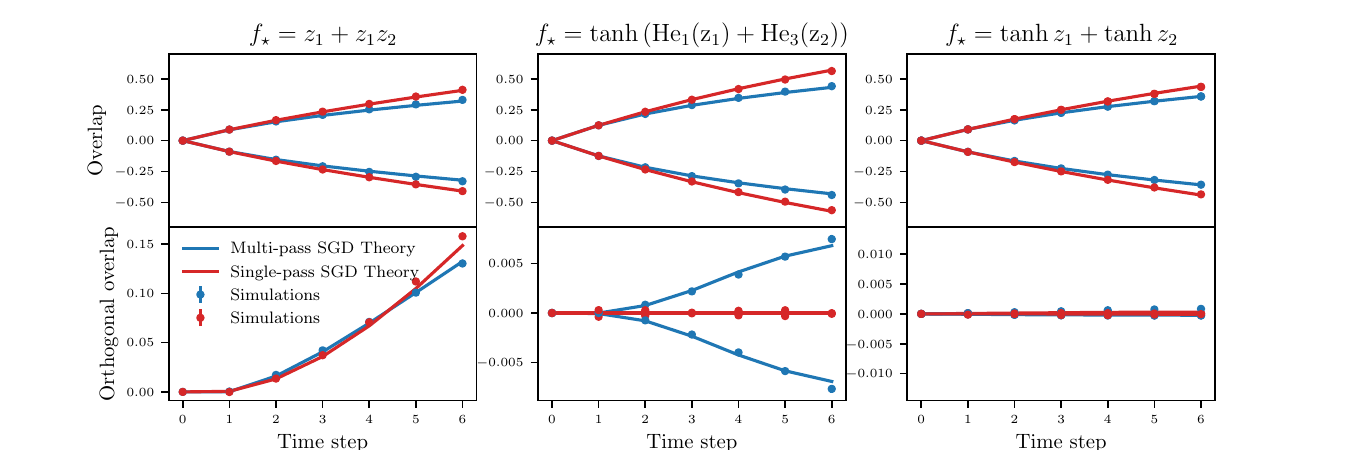}
\vspace{-2em}
    \caption{
    \textbf{One-pass and multi-pass GD for multi-index models --} 
The \textit{overlaps} between the student weights along the first direction learned, namely \(\mathbf{C}[f^\star]\), and its orthogonal, is plotted versus the number of iterations for three different classes of functions.     \textbf{Left: Easy finite-T learnable multi-index target} both the algorithms learn all the relevant directions when an "easy" function is used as a target (here (\(p=8\))).
     \textbf{Center: Multi-pass finite-T learnable multi-index target} both the algorithms learn the first Hermite direction \(\mathbf{C}[f^\star]\) but only multi-pass SGD achieve a non-null correlation in the orthogonal. This illustrates how reusing samples allows us to surpass the staircase limitation of single-pass approaches (\(p=2\)).
     \textbf{Right: Finite-time non-learnable multi-index target} neither of the two algorithm can learn \(\mathbf{C}[f^\star]^\bot\) with this target (\(p=2\)).
     (Simulation are averaged over $32$ runs,     $d\!=\!5000$, with $\sigma \!=\! \rm relu$, $n=3d$, $\eta \!=\! 0.1$).
    }
    \label{fig:main:multiple_index}
\end{figure*}

\section{Statement of the results}
\label{sec:main:theory_results}
Here, we introduce the main results covering the theoretical learning guarantees with gradient descent and contrast them with the known one-pass results. We exploit the rigorous DMFT construction to prove the first key result: two-layer networks efficiently learn a large class of multi-index targets in only $T\!=\!2$ iterations, breaking the curse of one-pass algorithms dictated by the information and leap exponents. 

\subsection{Finite-T Learnable and Non-learnable directions}
We first identify which target directions are hard to learn for multi-pass gradient descent. Define $U^\star$ to be the subspace spanned by the rows of the target weights $W^\star$. The ``hard" directions are the ones where any transformation of the output $f^\star(\vec{z})$ does not lead to a linear correlation along the direction. 
We now define the subspace of such directions:

\begin{definition}\label{def:two_step_hard}
    We define $P^\star$ as the subspace of directions $\vec{v}^\star\in U^\star$ such that for any polynomial $F:\R \rightarrow \R$ with coefficients in $\R$, the following condition is satisfied:
    \begin{equation}\label{eq:f_hard_gen}
        \Eb{\vec{z}}{F(f^\star(\vec{z}))\langle \vec{v}^\star, \vec{z}\rangle}=0,
    \end{equation}
    Similarly, we denote by $A^\star$, the subspace of directions where the above condition is satisfied for all real-valued analytic functions $F$.
\end{definition}
One part of our main result shows that directions in $A^\star$ cannot be learned even by re-using batches of size $\cO(d)$ in a finite number of gradient steps. Furthermore, under suitable conditions on $\sigma$ and $\vec{a}$ (discussed in Theorem \ref{thm:main:2step_learning} and Appendix \ref{app:sec:main_thm_proof}), we show that after two gradient steps, the first layer learns all directions in the complement $P^\star_\perp$. We are now ready to state our main result: 

\begin{theorem}
\label{thm:main:2step_learning}
 Suppose that $n/d=\alpha>0$. Let $\vec{v}^\star\in P^\star_\perp$ denote an arbitrary direction in the orthogonal complement of the subspace $P^\star$ defined in definition \ref{def:two_step_hard} with norm $\sqrt{d}$ and a fixed representation in the basis $W^\star$.
Suppose further that the activation function $\sigma$ is analytic, with polynomially bounded derivatives satisfying $\Eb{z \sim \mathcal{N}(0,1)}{\sigma'(z)} \neq 0$ and $\sigma^k(0) \neq 0  \ \forall k \in \N$. Then, for any $g^\star$ with derivatives bounded by polynomials, there exist $\eta> 0, \lambda > 0$ such that almost surely over the choice of $\vec{a}$, we have:
\begin{equation}
    \norm{\frac{\vec{W}^{(2)}\vec v^\star}{d}} = \Theta_d(1) \,,
\end{equation}
with high probability as $n,d \rightarrow \infty$. Furthermore, for large enough $p$, $\vec{W}^{(2)}$ asymptotically spans $P^\star_\perp$:
\begin{equation}
    \inf_{\vec v^\star \in P^\star_\perp}\norm{\frac{\vec{W}^{(2)}\vec v^\star}{d}} = \Theta_d(1)\,,
\end{equation}
with high probability as $n,d \rightarrow \infty$.
In other words, directions $\vec v^\star \in P^\star_\perp$ are learned in $T=2$ gradient steps.  

Suppose, however that the teacher subspace $U^\star\!=\!A^\star$, then:
\begin{equation}
    \sup_{\vec v^\star \in U^\star} \norm{\frac{\vec{W}^{(t)} \vec v^\star}{d}} = o_d(1)\,,
\end{equation}
with high probability as $n,d\! \rightarrow\! \infty$, for any finite time $t$. Thus, none of the directions are learned in any finite number of GD steps.\looseness=-1
\end{theorem}
The proof is based on the analysis of the DMFT equations discussed
in Sec. \ref{sec:main:asymptotics}, is given in App.~\ref{sec:app:proofs}, and we provide an informal heuristic derivation in sec. \ref{sec:hidden_prog}.
While the above negative result requires all directions in $U^\star$  to be in $A^\star$ and thus in $P^\star$, in App \ref{sec:app:staircase}, we discuss the more general setup where learning of certain directions in $P^\star_\perp$ can affect the learning of directions in $U^\star$ in subsequent timesteps.

When the expectation in Equation \ref{eq:f_hard_gen} is non-zero for $F$ being the identity mapping, i.e. $F = \mathrm{id}$, $\vec{v}^\star$ is in-fact learned in the first gradient step \citep{ba2022high,dandi2023twolayer} or through online SGD \citep{ben2022high,abbe2023sgd}. We discuss this further in Section \ref{sec:com_multi_vs_once}.

Our analysis reveals that the effect of re-using batches is to implicitly transform the output in the subsequent steps, allowing a larger set of directions to be learned. However, for directions in $A^\star$, such transformations are still insufficient.

\subsection{Characterization of hard directions through symmetries}

While Definition \ref{def:two_step_hard} characterizes the subspace of hard directions 
$A^\star$, it requires checking that the equality in Equation \ref{eq:f_hard_gen} holds for any real analytic transformation $F$.
We now show that a sufficient condition for $\vec{v}^\star \in A^\star$ is for $f^\star$ to possess certain symmetries along  $\vec{v}^\star$. This leads us to identify subspaces of hard directions, contained in $A^\star$, linked to symmetries w.r.t certain transformations. We characterize such subspaces below. The simplest such symmetry is defined through reflection along $\vec{v}^\star$:

\begin{definition}\label{def:symmetric_subspace}
   For any direction $\vec{v}^\star \neq 0 \in \R^d$, let $R_{\vec{v}^\star}$ denote the reflection operator along $\vec{v}^\star$, i.e. $R_{\vec{v}^\star}=\vec{I}-2\frac{1}{\norm{\vec v^\star}^2}\vec{v^\star}\vec{v^\star}^T$. We say that a direction $\vec{v}^\star \in U^\star$ is \text{even-symmetric} w.r.t $f^\star$ if for any $\vec{z} \in \R^d$:
    \begin{equation}
        f^\star(R_{\vec{v}^\star}\vec{z})= f^\star(\vec{z})
    \end{equation}
    We denote by $E^\star$ the subspace spanned by all \text{even-symmetric} directions in $U^\star$.
\end{definition} 
It is straightforward to see that any $\vec{v}^\star \in E^\star$ leads to Equation \ref{eq:f_hard_gen} being satisfied for any transformation $F$, since $\vec{v}^\star$ remains even w.r.t the function $F\left(f^\star (\cdot)\right)$. Therefore, $E^\star \subseteq A^\star$
However, the set of non-learnable directions can be larger due to the presence of additional symmetries. We now define such a larger subspace of hard directions arising due to a symmetry w.r.t reflections along $\vec{v}^\star$ coupled with orthogonal transformations along the orthogonal subspace:

\begin{definition}\label{def:ortho_symmetric_subspace}
   For any direction $\vec{v}^\star \neq 0$ in $U^\star$, let $R_{\vec{v}^\star}$ be as defined in Definition \ref{def:symmetric_subspace}. Let $O_{\perp}$ be a matrix in the orthogonal group on the $d-1$ dimensional subspace $\{\vec v^\star\}_\perp$ i.e the orthogonal complement of the linear subspace spanned by $\vec{v}^\star$.
        We say that a direction $\vec{v}^\star$ is \text{orthogonally-even-symmetric} w.r.t $f^\star$, if there exists an $O_{\perp} \in O(\{v^\star\}_\perp)$, such that for any $\vec{z} \in \R^d$:
    \begin{equation}
        f^\star(O_{\perp}R_{\vec{v}^\star} \vec{z})= f^\star(\vec{z})
    \end{equation}
    We denote by $OE^\star$ the subspace spanned by all \text{orthogonally-even-symmetric} directions in $U^\star$.
\end{definition} 

By setting $O_{\perp}$ as the identity mapping in the above definition, we recover the condition for $\vec{v}^\star \in E^\star$. Therefore, we have that $E^\star \subseteq OE^\star$.
While $OE^\star$ is the largest set of directions we've identified as being hard, the true set of hard directions may be larger still and is given by $A^\star$ in Definition \ref{def:two_step_hard}. We show in Appendix \ref{app:sec:proof:OE} that the directions in $OE^\star$ are indeed hard as per Definition \ref{def:two_step_hard}:
\begin{proposition}\label{prop:OE}
Let the subspaces $A^\star$ and $OE^\star$  be as defined in Defs. \ref{def:two_step_hard} and \ref{def:ortho_symmetric_subspace} respectively. Then, $OE^\star \subseteq A^\star$ i.e. all directions in $E^\star, OE^\star$ are hard as per Thm. \ref{thm:main:2step_learning}.
\end{proposition}

App.\ref{sec:app:even_sym} gives several examples  where $P^\star_\perp = E_\perp^\star = U^\star$, such as single-index targets with odd Hermite activations, staircase functions, etc.
Interestingly, we show that there exist functions $f^\star$ where the set $OE^\star$ is strictly larger than $E^\star$. Consequently, for such functions, $E^\star$ is strictly contained in  $A^\star,P^\star$. We discuss such target functions in Appendix \ref{sec:app:hard_sym}.
 For example, we show in Appendix \ref{sec:app:hard_sym}, that for the target function $f^\star(\vec{z})= z_1z_2z_3$, the direction $\vec{v}^\star=\vec{e}_1+\vec{e}_2+\vec{e}_3$ does not lie in $E^\star$ but lies in $OE^\star$ and thus in $A^\star,P^\star$.
 \vspace{-0.2cm}

\subsection{Comparison between one-pass and multi-pass GD}\label{sec:com_multi_vs_once}


Our results are particularly interesting in the context of a recent line of work on the limitations of one-pass algorithms. 
\citep{BenArous2021,abbe2021staircase,abbe2022merged,abbe2023sgd, dandi2023twolayer, bietti2023learning,zweig2023symmetric}. We can demonstrate, in particular, a sharp separation performance between one-pass and multiple-pass protocols.

\paragraph{Learning single-index targets --} 
First, we consider single index targets.
Targets that are hard to learn 
for one-pass algorithms starting from uninformed initialization in high dimension are characterized by the \textit{Information Exponent} ($\rm{IE}$). Informally, the $\rm{IE}$ is equivalent to the first non-zero coefficient in the Hermite expansion of the target activation. 
\begin{definition}[Information Exponent]
\label{def:main:information_exponent}
\citep{BenArous2021}
Let $\mathrm{He}_j$ be the $j-$th Hermite polynomial. Using the definition for the target of eq.~\eqref{eq:main:teacher_def},  reading in the single-index case as $f^\star = g^\star(\langle \vec w^\star, \vec z \rangle)$,  the $\rm IE$ is defined as:
\begin{align}
{\rm{IE}} = \min \{j\in\mathbb{N}: \mathbb{E}_{\xi \sim \mathcal{N}(0,1)} \left[g^{\star}(\xi) \mathrm{He}_j(\xi)  \right]\neq 0\} 
\end{align}
\end{definition}
Higher $\rm IEs$ are associated to harder problems for one-pass training protocols. Indeed, \citep{BenArous2021} provably show that one-pass SGD, with one sample per batch, weakly recovers the teacher direction $\vec w^\star$ only upon iterating the training schedule for $T(\rm IE)$ time iterations:
\begin{align}
\label{eq:main:time_complexity_IE}
    T(\rm IE) = \begin{cases}
        &\cO(d^{\rm IE -1} ) \qquad \hspace{0.3em}\text{if $\rm{IE}>2$} \\
        &\cO(d \log{d} ) \qquad \text{if $\rm{IE}=2$} \\
        &\cO(d) \qquad \hspace{2.2em}\text{if $\rm{IE}=1$}
    \end{cases}
\end{align}
Recently, the time complexity has been improved up to the Correlational Statistical Query (CSQ) lower bound of $d^{\max{(\frac{\rm{IE}}{2}},1)}$, by considering an appropriate smoothing of the loss \citep{damian_2023_smoothing}.
Definition~\ref{def:main:information_exponent} has been extended to larger batch sizes in \citep{abbe2022merged,dandi2023twolayer}, without changing the overall picture; more precisely, even with $n = o(d^{\rm{IE}})$ fresh samples per batch, one-pass training procedures are still not able to weakly recover the signal in finite iteration time. The case $\rm IE=1$ corresponds to the expectation in Equation \ref{eq:f_hard_gen} being non-zero for $F=\mathrm{id}$. However, since Definition \ref{def:two_step_hard} allows for general transformations $F$ to the output $f^\star$, $\vec{w}^\star$ may not be in $P^\star, A^\star$ even when $\rm{IE} > 1$. The presence of general transformations $F$ in definition \ref{def:two_step_hard} allows our algorithm to bypass CSQ bounds, which are restricted to $F=\mathrm{id}$. Such general transformations are however permitted under the framework of Statistical Query (SQ) algorithms \citep{kearns1998efficient}. We thus expect gradient descent with  $\cO(d)$ sample complexity to inherit the hardness results established for the class of SQ algorithms \citep{diakonikolas2020algorithms, goel2020superpolynomial, chen2021efficiently, chen2022hardness}.
We emphasize however that unlike explicit SQ algorithms, our analysis shows that gradient descent performs such transformations implicitly, allowing it to reach the optimal complexity of SQ algorithms for certain class of target functions.
\looseness=-1

We illustrate the sharp contrast between one-pass and multiple-pass protocols with the examples depicted in Figure~\ref{fig:main:single_index}, which shows the scalar product (called \textit{overlap}) between the learned weights and the teacher direction as a function of the time steps and compares simulation (dots) with theoretical predictions (continuous lines). There are $3$ cases:
\vspace{-0.3cm}
\begin{itemize}[noitemsep,leftmargin=1em,wide=0pt]
    \item \textbf{Easy finite-$T$ learnable single-index targets $\left(\rm{IE}\!=\!1\right)$:} The left panel of Fig.~\ref{fig:main:single_index} show the learning curve for a problem with $\rm{IE} \!=\! 1$. {\it Both} single-pass and multiple-pass GD correlates with the target in finite time. The non-symmetric subspace coincides with the teacher one $E_\perp^\star = U^\star$ (Def.~\ref{def:symmetric_subspace}).
    \item \textbf{Multi-pass finite-$T$ learnable single-index targets $\left(\rm IE\!>\!1, \rm {non-even \,\,f^\star} \right)$: }  Fig.~\ref{fig:main:single_index} (center) depicts the learning curve for a non-even target function, with $\rm{IE} = 3$. Here, one-pass GD is not able to achieve any significant correlation with the teacher $\vec w^\star$ (and it would require a number of iterations $T = \cO(d)$ to achieve weak recovery - see eq.~\eqref{eq:main:time_complexity_IE}). However, multiple-pass GD performs weak recovery in only $T=2$ steps. As before, the non-symmetric subspace corresponds to the teacher one $E_\perp^\star = U^\star$ (Def.~\ref{def:symmetric_subspace}).
    \item \textbf{Finite-$T$ non-learnable single-index targets $\left(\rm{IE}\!>\!1, \, \, even \,\, f^\star\right)$:} Fig.~\ref{fig:main:single_index} (right) considers an even problem, with $\rm IE = 4$. Neither of the training procedures achieve weak recovery  in finite time. The computational hardness of this problem is associated with the presence of symmetry in the teacher function that requires time to break. Indeed, following Definition~\ref{def:symmetric_subspace}, the even-symmetric subspace $E^\star$ is equivalent to the teacher subspace $U^\star$. These results agree with the emergence of \textit{computational barriers } in symmetric single-index problems like the phase retrieval one \citep{maillard2020phase}. In fact, for such problems, regardless of the number of iterations, learnability requires $\alpha$ to be larger than critical values even for the most efficient known algorithms (see \citep{barbier2019optimal}, Sec. 3.1).
    \end{itemize}

\vspace{-0.5cm}

\paragraph{Learning multi-index targets -- } The hardness of 
multi-index targets learning has been the subject of numerous recent studies for single-pass algorithms \citep{abbe2021staircase,abbe2022merged,abbe2023sgd,bietti2023learning, zweig2023symmetric, dandi2023twolayer}. The class of multi-index targets efficiently learned by one-pass algorithms has been provably associated with the \textit{Leap Complexity} ($\rm LC$) of the target to be learned, which generalizes the information exponent: 
\begin{remark}
    To enhance the clarity of the presentation, we limit the definition of the $\rm LC$ to an informal one. We refer to Section~$\bf B.2$ (\textit{isoLeap}) in \citep{abbe2023sgd} and Definition~$\bf 3$ in \citep{dandi2023twolayer} (\textit{leap index}) for details.  
\end{remark}
Informally, the learning dynamics of one-pass routines follow this behavior: initially, the network learns in the first step the first Hermite coefficient of the target $f^\star$. For every time $t\!\in\! [T]$ of the one-pass schedule, the network is bound to learn in finite time only features that are \textit{linearly connected} to the previously learned directions; functions possessing only such linearly connected features are leap $1$ functions ($\rm LC =1$), e.g. $f^\star(\vec z) \!=\! z_1 \!+\! z_1z_2$. Similarly, functions that are quadratically connected to the learned features are leap $2$  ($\rm LC \!=\! 2$), e.g. $f^\star (\vec z) \!=\! z_1 + z_1z_2z_3$.
 Higher $\rm{LC}$ target functions correspond to harder learning problems for one-pass algorithms:  one-pass SGD, with one sample per batch, weakly recovers the teacher subspace $U^\star$ by iterating the training protocol for $T(\rm{LC})$ time steps, where the $\rm{LC}$ substitutes the $\rm{IE}$ in eq.~\eqref{eq:main:time_complexity_IE} \citep{abbe2023sgd}.

We illustrate the behavior of one-pass and multiple-pass algorithms when learning multi-index functions in Fig.~\ref{fig:main:multiple_index}. Using different two-index teachers ($k=2$), it shows the scalar product between the learned weights and two reference vectors: a) the first Hermite coefficient of the target $f^\star$, called in the following $\vec{C}_1[f^\star]$; b) the vector in the teacher subspace $U^\star$ orthogonal to $\vec{C}_1[f^\star]$, referred as $\vec{C}_1[f^\star]^\perp$. The figure exemplifies the correlations metrics as a function of time, labeled as \textit{overlap} (resp. \textit{orthogonal overlap}) in the upper (resp. lower) section. There are, again, 3 cases:
\begin{itemize}[noitemsep,leftmargin=1em,wide=0pt]
    \item  \textbf{Finite-$T$ learnable multi-index targets:} Fig.~\ref{fig:main:multiple_index} (left)  depicts a target with $\rm LC \!=\!1$.  The teacher subspace $U^\star$ spanned by the standard basis vectors $\{\vec e_1, \vec e_2\}$ is learned by both one-pass and multi-pass GD in finite time. At $T\!=\!1$, $\vec e_1\!=\!\vec{C}_1[f^\star]$ is learned; this enables the recovery of the direction $\vec e_2$ at $T\!=\!2$ as the target is linear in $z_2$ once $e_1$ has been learned. This hierarchical picture of learning is called \textit{staircase} mechanism. Using Def.~\ref{def:symmetric_subspace} notations, the non-symmetric teacher subspace $E^\star_\perp$ is equivalent to the full teacher subspace $U^\star$.
    \item \textbf{Multi-pass finite-$T$ learnable multi-index targets:}
    The central panel in Fig.~\ref{fig:main:multiple_index} illustrates a teacher with $\rm LC =3$. Both algorithms are successful in weakly recovering the direction $\vec C_1[f^\star]$ in the first step. However, as the training continues, one-pass GD never recovers the full teacher subspace in finite time (exemplified by the zero orthogonal overlap in the lower panel). Conversely, multi-pass GD is able to perform weak recovery of the full teacher subspace $U^\star$ by achieving a non-vanishing correlation with $\vec C_1[f^\star]^\perp$ (non-zero orthogonal overlap in the lower section) in just $T=2$ steps. Again, the non-symmetric subspace $E_\perp^\star = U^\star$  is equivalent to the full teacher subspace $U^\star$ (Def.~\ref{def:symmetric_subspace}).
    \item \textbf{Finite-$T$ non-learnable multi-index targets:} The right panel of Fig.~\ref{fig:main:multiple_index} considers a committee machine teacher with symmetric activation, i.e. $f^\star(\vec z) = \sum_{r=1}^2 \sigma^\star \left(\langle \vec z, \vec e_r\rangle \right)$, here $\rm LC = 2$. 
    Both protocols, in this case, are only able to learn a single-index approximation of the target function in finite time, achieving non-zero correlation only with $\vec C_1[f^\star] \propto \vec e_1 + \vec e_2$  throughout the dynamics. The computational hardness of this problem is associated with the presence of a neuron exchange symmetry. Indeed,  using Def.~\ref{def:symmetric_subspace} notations, we observe that the even-symmetric subspace  $E^\star = \{\frac{1}{\sqrt{2}} \left(\vec e_2 - \vec e_1 \right)\}$ is a non-empty subspace of the teacher one $U^\star$. Therefore, as for one-pass routines, multiple-pass ones are bound to learn only $\vec{v}^\star =  \left(\vec e_1 + \vec e_2 \right)/\sqrt{2}$ in finite time steps. 
    Such difficulties have been described in the analysis of the \textit{specialization}  transition in the information-theoretic/Bayes optimal case of symmetric committees \citep{Aubin_2019}. As for single index models, breaking the symmetry requires $\alpha$ to be large enough and, even in this case, the best-known algorithms require a diverging number of iterations (see \citep{Aubin_2019}, Sec. 3).
\end{itemize}
\vspace{-0.4cm}

\subsection{From weak recovery to generalization}
While Th.~\ref{thm:main:2step_learning} provides conditions for the {\it weak recovery} (a finite overlap with directions in $U^\star$), once this is done, it becomes straightforward to learn the function up to any desired accuracy with only $\cO(d)$ additional samples.  Indeed, strong generalization guarantees can be proven by utilizing existing results either for subsequent training with online SGD \citep{BenArous2021} (to use their terminology, once you escape {\it mediocrity}, the ballistic phase is {\it easy}) or training of the second layer using an independent batch of $\cO(d)$ samples as in \citep{damian2022neural,abbe2023sgd}. See App.\ref{sec:app:gen_error} for such generalization sample-complexity results.\looseness=-1
\vspace{-0.2cm}

\section{Main proof ideas}

\subsection{Learning by hidden progress: heuristic argument}\label{sec:hidden_prog}
While we give a rigorous proof of Thm. ~\ref{thm:main:2step_learning} in App.~\ref{sec:app:proofs}, we provide now an informal description of the \textit{hidden progress}  in the first step of gradient descent that allows subsequent development of overlaps in the second step, that is at the root of the difference between single and multi-pass algorithms.
For simplicity, we focus on the case of a single hidden neuron ($p\!=\!1$). 
We denote  $h^{(t)}_\mu\!=\!\langle \vec w^{(t)}, \vec z_\mu \rangle / \sqrt{d}$ the pre-activation for the $\nu_{th}$ training point along the neuron  with $\mu \in [n]$, and $\vec{v}^\star$ a vector in the span of $\vec{W}^\star$ with $\norm{\vec{v}^\star}\!=\!\sqrt{d}$.


From the gradient update in Eq.~(\ref{eq:GD_update}), 
the update lies in the span of the training inputs $\{ \vec{z}_\nu\}_{\nu=1}^n$, with the gradient of the $\nu_{th}$ training example given by $\vec{g}_\nu=a\mathcal{L}'\left(h^{(t)}_\nu,f^\star(\vec z_\nu)\right)\sigma'(h^{(t)}_\nu)\vec z_\nu/\sqrt{d}$. For squared loss, assuming that $f\left( \frac{W^{(t)}\vec z_\nu} {\sqrt{d}}\right) \approx 0$, the gradient reads:
\begin{equation}\label{eq:update_one_example}
    \vec{g}^{(t)}_\nu \approx -a f^\star(\vec z_\nu)\sigma'(h^{(t)}_\nu)\vec z_\nu / \sqrt{d}\,.
\end{equation}
At initialization $h^{(0)}_\nu=\langle \vec w^{(0)}, \vec z_\nu \rangle / \sqrt{d}$ and the projections along the teacher subspace (which we denote $\vec{h}^\star_{\nu}=\frac{1}{\sqrt{d}} W^\star \vec{z}_\nu \in \R^k$) are approximately independent since $\vec w^{(0)}$ is approximately orthogonal to the teacher subspace as well as to the inputs  $\{ \vec{z}_\nu\}_{\nu=1}^n$. The projection of the gradient along the teacher subspace is given by:
\begin{align}
   W^\star\left(\sum_{\nu=1}^n \frac{\vec{g}^{(t)}_\nu}{\sqrt{d}} \right) \!\approx\! \!-\!\sum_{\nu=1}^n a f^\star(\vec z_\nu)\sigma'(h^{(t)}_\nu) \frac{\vec{h}^\star_{\nu}}{d} \,.
\end{align}
We do expect that, due to concentration, the component of the full-batch gradient update along the teacher subspace lies along the direction given by:
\begin{equation}\label{eq:update_first_step}
    \Ea{f^\star(\vec z_\nu)\sigma'(h^{(0)}_\nu)\vec{h}^\star_{\nu}}\!\approx \!\Ea{\sigma'(h^{(0)}_\nu)}\!\Ea{f^\star(\vec z_\nu)\vec{h}^\star_{\nu}},
\end{equation}
where we used the approximate independence of $h^{(0)}_\nu$ and  $\vec{h}^\star_{\nu}$ to factorize the expectation.
Thus, the neuron parameters $\vec w^{(1)}$ at the first step are correlated with the teacher subspace only along the direction $\Ea{f^\star(\vec z_\nu)\vec{h}^\star_{\nu}}$. 

If $\Ea{f^\star(\vec z_\nu)\vec{h}^\star_{\nu}}=0$, the parameters remain orthogonal to the teacher subspace. This is true whenever the $\rm LC$ of the target function is larger than $1$. To make progress, it is thus necessary for the pre-activations $h^{(t)}_\nu=\langle \vec w^{(t)}, \vec z_\nu \rangle / \sqrt{d}$ to become correlated with the teacher pre-activation $\vec{h}^\star_{\nu}$. This can happen in two different ways:
   
    (i) By $\vec{w}^{(t)}$ directly gaining components along the teacher subspace $W^*$. Under online SGD, the data is used only once for the gradient updates, so only this mechanism is possible. It allows the directions learned by $\vec{w}^{(t)}$ at any step to depend on the directions already learned by $\vec{w}^{(t)}$. This underlies the ``staircase" phenomenon in online SGD \citep{abbe2021staircase,abbe2022merged,abbe2023sgd} as well as the notion of information exponent when applied to a single direction \citep{BenArous2021}.
    
    (ii) By $\vec{w}^{(t)}$ gaining components along $\vec{z}_\nu$. Recall that the target is defined as $y_\nu=f^{\star}(\vec{z}_\nu)\,,$ and thus $\vec{w}^{(t)}$ can correlate with $\vec{h}^\star_{\nu}$. This is what happens when using gradient descent with multi-pass in our setting. This implies that even when $\vec w^{(1)}$ does not learn a direction $\vec{v}_*$, the pre-activation $h^{(t)}_\nu$ can develop a dependence on $\vec{h}^\star_{\nu}$ through the component of the gradient update along $\vec{z}_\nu$. 

Let us see how this phenomenon, which we call \textit{hidden progress},  happens in practice.  From (\ref{eq:GD_update}), the update in the pre-activation $h^{(1)}_\nu$ due to the first gradient step reads:
\begin{align}
    h^{(1)}_\nu\!=\!(1&-\eta \lambda) h^{(0)}_\nu  -\eta  \frac{a}{d} \sum_{{\nu'}=1}^{n}\mathcal{L}'(h^{(0)}_{\nu'},f^\star(\vec z_{\nu'}))\sigma'(h^{(0)}_{\nu'})\langle \vec z_{\nu'}, \vec z_\nu \rangle\,. 
\end{align}
In this sum there is one term of magnitude $\cO\left({\norm{\vec z_\nu}^2}/{d}\right)=O(1)$ corresponding to $\nu'\!=\!\nu$, and $d-1$ random terms of  order $\cO\left({\langle \vec z^\nu, \vec z^\nu \rangle}{/d}\right)=O\left(1/{\sqrt{d}}\right)$. This second group of terms contributes to an effective ``noise" of order $O(1)$. The first term however, since $\|{\vec z}_{\nu} \|_2^2 \approx d$, depends on $f^\star(\vec z_\nu)$ (and thus on all components of $\vec{h}^\star_{\nu}$):
\begin{equation}\label{eq:hid_prog_term}
    \mathcal{L}'(h^{(0)}_\nu,f^\star(\vec z_\nu))\sigma'({h^{(0)}_\nu})\,.
\end{equation}
Due to this dependence between $h^{(1)}_\nu$ and $f^\star(\vec{z}_{\nu})$, in the subsequent steps i.e. $T\!=\!2$, the term $\sigma'(h^{(1)}_\nu)$ in the update (\ref{eq:update_one_example}) can now influence the direction of the gradient along the teacher subspace, leading to $\vec{w}^{(2)}$ gaining correlations with new directions in $W^\star$.  It can be seen as follow: let $m_{\vec{v}_\star}^{(t)} =\langle \vec w^{(t)}, \vec{v}^\star \rangle / \sqrt{d}$, it follows from the GD updates that 
\begin{align}
    m_{\vec{v}^\star}^{(2)} \!=\! (1&-\eta\lambda)m_{\vec{v}^\star}^{(1)} -\eta  \frac{a_j}{d} \sum_{\nu=1}^{n}\mathcal{L}'(h^{(1)}_\nu,f^\star(\vec z_\nu))\sigma'(h^{(1)}_\nu)h^{\vec{v^\star}}_{\nu}
\end{align}
Now, suppose that $\vec{v^\star}$ is not learned in the first step.
However, due to the hidden progress, $h^{(1)}_\nu$ is now dependent on $h^{\vec{v}^\star}_{\nu} = \langle\vec{v}^\star, \vec{z}_\nu \rangle/\sqrt{d}$, thus allowing the new expectation of the projection of the update along $\vec{v^\star}$ given by $\Ea{f^\star(\vec z_\nu)\sigma'(h^{(1)}_\nu)h^{\vec{v^\star}}_{\nu}}$ to be non-zero. This explains how the dependence of the pre-activations $h^{(t)}_\nu$ on $f^\star(\vec z^\nu)$ can allow learning of new directions even when the weights have not gained components along the teacher subspace.\looseness=-1


This learning mechanism, however, fails when the target function is symmetric along $\vec{v^\star}$. Indeed, for such a direction, $h^{(1)}_\nu$ retains an {\it even} dependence on $h^{\vec{v}^\star}_{\nu}$, which implies that the expectation of the term 
$\Ea{f^\star(\vec z_\nu)\sigma'(h^{(t)}_\nu)h^{\vec{v^\star}}_{\nu}}$
remains $0$ for all time steps $t \in [T]$, with $T =\cO(1)$. Such directions are therefore not learned with a finite number of time-steps and batch-size $n=O(d)$ even upon re-using the batches.\looseness=-1
The rigorous control of all these quantities is a difficult task a priori. One cannot, in particular, express the above sum as an expectation w.r.t independent samples $\vec z_\nu$ since the weights $W^{(1)}$ now depend on all the samples. Fortunately, this is precisely the difficulty solved by the DMFT equations through an effective stochastic process on the pre-activations that are decoupled across training examples. The rigorous analysis is detailed in App.~\ref{sec:app:proofs}. The main lines of the DMFT equations are in Sec.~\ref{sec:main:asymptotics}.\looseness=-1

Finally, note that while our proof uses the Gaussian data assumption, the heuristic argument hints that this is not crucial. Additionally, in any {\it real} dataset samples are very correlated, and thus a given sample (or a very similar one) may appear many times. In this case, even {\it single-pass} algorithms will behave as predicted by our approach. We thus believe it describes a more realistic scenario than the pure single pass theories with fresh {\it i.i.d.} data.

%% file: sections/asymptotics.tex
\subsection{Characterization of the dynamics}

\label{sec:main:asymptotics}
Re-using batches at each gradient step requires keeping track of the pre-activations of the parameters. Since  the number of pre-activations and the dimensions of the parameters grows with $d$, we need a low-dimensional effective dynamics characterizing the quantities of interests such as the overlaps between the student and target parameters. DMFT provides such an effective dynamics through a set of coupled stochastic processes $\boldsymbol \theta^{(t)} \in \R^p$ and $\vec h^{(t)} \in \R^p$ representing the joint-distributions of the student, teacher parameters  $W^{(t)}, W^\star$. and the student, teacher pre-activations respectively.\looseness=-1

We derive the equations and prove their applicability to our setting using existing results in \citep{celentano2021highdimensional, gerbelot2023rigorous}. Asymptotically, for $d\!\to\!\infty$ with $n\!=\!\alpha d$, the joint distribution of the student and teacher pre-activations (for each sample), $\vec h^{(t)}_\nu \!=\! W^{(t)} \vec z_\nu / \sqrt{d} \!\in\! \mathbb{R}^p$ and  $\vec h^\star_\nu \!=\! W^\star \vec z_\nu / \sqrt{d} \!\in\! \mathbb{R}^k$ converge in distribution to samples from the stochastic process $\vec h^{(t)}$ and the standard normal variable $\vec h^\star$. Similarly, the joint distribution of each component of the student and teacher weights $W^{(t)}_i \in \mathbb{R}^p,W^\star_i \in \mathbb{R}^k$ with $i \in [d]$ converge in distribution to samples from the stochastic process $\boldsymbol \theta^{(t)}$ and the standard normal variable $\boldsymbol\theta^\star$.
\begin{align}
\label{eq:def_processes_main}
&\boldsymbol{\theta}^{(t+1)} = \left(1- \eta\lambda -\eta\Lambda^{(t)}\right)\boldsymbol{\theta}^{(t)} 
    + \,\eta \sum_{\tau=0}^{t-1}  R_\mathcal{L}^{(t,\tau)}\boldsymbol{\theta}^{(\tau)} - \,\eta g^{(t)}\boldsymbol{\theta}^\star + \,\eta \sum_{\tau=0}^{t-1}  \tilde R_\mathcal{L}^{(t,\tau)}\boldsymbol{\theta}^\star + \,\eta \boldsymbol{u}^{(t)}\\ \label{eq:def_processes_main_2}
    &\boldsymbol{h}^{(t)} = -\eta\sum_{\tau=0}^{t-1}R_\theta^{(t,\tau)} \nabla_{\boldsymbol{h}} \mathcal{L}(\boldsymbol{h}^{(\tau)}, \boldsymbol{h}^\star) + \boldsymbol{\omega}^{(t)}
\end{align}
Notice that the formula above is the high dimensional equivalent of the gradient descent update \eqref{eq:GD_update}. Here $\vec u^{(t)}$ and $(\boldsymbol \omega^{(t)}, \boldsymbol \theta^\star)$ are zero mean Gaussian Process with covariances $C_\mathcal{L}^{(t,\tau)}$ and $\Omega^{(t,\tau)}$ respectively, with
\begin{align} 
    C_\mathcal{L}^{(t, \tau)} \!\!= 
    \alpha \mathbb{E}_{\boldsymbol{h}^{(t)}, \boldsymbol{h}^\star}\!\left[ \nabla_{\boldsymbol{h}} \mathcal{L}(\boldsymbol{h}^{(t)}, \boldsymbol{h}^\star)\nabla_{\boldsymbol{h}} \mathcal{L}(\boldsymbol{h}^{(\tau)}, \boldsymbol{h}^\star)^\top\!\right] \nonumber \\
    \Omega^{(t,\tau)} \!\!= \!\! \begin{bmatrix}
        C_\theta^{(t,\tau)} & M^{(t)} \\
        M^{(\tau)} & 1
    \end{bmatrix} \!\!=\!  \mathbb{E}_{\boldsymbol{\theta}^{(t)}, \boldsymbol{\theta}^\star}\!\left[ \begin{pmatrix}
        \boldsymbol{\theta}^{(t)} \\ \!\! \boldsymbol{\theta}^\star
    \end{pmatrix} \begin{pmatrix}
        \boldsymbol{\theta}^{(\tau)} \!  \\ \boldsymbol{\theta}^\star \!
    \end{pmatrix}^\top \!\right] \nonumber
\end{align}
the matrix $\Lambda^{(t)}$ can be viewed as an ``effective regularization" on the parameters. $\Lambda^{(t)}$ and the projected gradient $g^{(t)}$ converge in probability to:
\begin{align} \label{eq:lambda_def}
    \Lambda^{(t)} \!=\! \alpha\mathbb{E}_{\boldsymbol{h}^{(t)}, \boldsymbol{h}^\star}\left[\nabla^2_{\vec h}\mathcal{L}\left( \vec h^{(t)}, \boldsymbol{h}^\star\right)\right]\,,\,\\ g^{(t)}\!=\!\alpha \mathbb{E}_{\boldsymbol{h}^{(t)}, \boldsymbol{h}^\star}\left[\nabla_{\vec h}\mathcal{L}\left( \vec h^{(t)}, \boldsymbol{h}^\star\right) \vec h^{\star\top}\right]
\end{align}
The memory kernels $R_\mathcal{L}^{(t,\tau)}$, $\tilde R_\mathcal{L}^{(t,\tau)}$, $R_\theta^{(t,\tau)}$ are defined as: 
\begin{eqnarray} \label{eq:R_def}
    &R_\theta^{(t,\tau)} = \mathbb{E}_{\boldsymbol{\theta}^{(t)}, \boldsymbol{\theta}^\star}\left[\frac{\partial\, \vec \theta^{(t)}}{\partial\, \vec u^{(\tau)}} \right]\,, \nonumber  \\
    &R_\mathcal{L}^{(t,\tau)} = \alpha \mathbb{E}_{\boldsymbol{h}^{(t)}, \boldsymbol{h}^\star}\left[\frac{\partial\, \nabla_{\vec h}\mathcal{L}\left( \vec h^{(t)}, \boldsymbol{h}^\star\right)}{\partial\, \boldsymbol\omega^{(\tau)}} \right]\,,\\
    &\tilde R_\mathcal{L}^{(t,\tau)} = \alpha \mathbb{E}_{\boldsymbol{h}^{(t)}, \boldsymbol{h}^\star}\left[\frac{\partial\, \nabla_{\vec h}\mathcal{L}\left( \vec h^{(t)}, \boldsymbol{h}^\star\right)}{\partial\, \left(\boldsymbol\theta^\star\right)^{(\tau)}} \right]\,,\quad
\end{eqnarray}
and $R_\mathcal{L}^{(t,t)} = \tilde R_\mathcal{L}^{(t,t)} = 0$, $R_\theta^{(t,t)}=1$.
Finally, the low dimensional projections of the weights $M^{(t)}$ will obey
\begin{equation} \label{eq:M_update}
    M^{(t+1)} = (1-\eta\lambda)M^{(t)} -\eta g^{(t)} \, .
\end{equation}

Notice that these definitions are well-posed because of the causal structure of the gradient descent upgrades, and by extension of \eqref{eq:def_processes_main}: the distribution of $(\boldsymbol \theta^{(t+1)}, \vec h^{(t+1)})$ is completely determined by $\{ (\boldsymbol\theta^{(\tau)}, \vec h^{(\tau)})\}_{\tau \in [t]}$ and the auxiliary quantities in eqs.~(\ref{eq:lambda_def},~\ref{eq:R_def}). Iterating backwards we reach the initial condition $\boldsymbol \theta^{(0)}$, which is a simple function of the data distribution and the initial conditions of the weights. For additional details we refer to App. \ref{sec:app:proofs}. Notice that it is also possible to write this set of equations as a function of a single stochastic process on $\vec h$, as in App. \ref{sec:app:hyperparams}.\looseness=-1

\paragraph{Sketch of proof of the hidden progress ---} Finally, we explain how the DMFT equations relate to the phenomenon in Sec. \ref{sec:hidden_prog} and allow us to prove Th. \ref{thm:main:2step_learning}.
The term $\nabla_{\vec h^{(1)}}\mathcal{L}\left( \vec h^{(1)}\right)$ in (\ref{eq:def_processes_main_2}) precisely corresponds to the contribution to pre-activation of a point $\vec{z}^\nu$ (App. \ref{sec:app:proof_hidden_prog})
from the gradient at the same point $\vec{z}^\nu$. As we discussed in Section \ref{sec:hidden_prog}, this term induces a dependence between $\vec h^{(1)}_\nu$ and $\vec h^\star_\nu$ even when the overlaps $M^{(t)}$ are $0$. At time $T=1$, the response term simplifies to $R_\theta^{(1,0)}=\mathbf{I_d}$ and the pre-activations can be expressed as the random variable $\nabla_{\vec h^{(t)}}\mathcal{L}\left( \vec h^{(t)}\right)$ with added Gaussian noise.
Analogous to section \ref{sec:hidden_prog}, we denote by $M_{\vec{v}^*}^{(t)}$ the limiting value of the overlaps $\frac{1}{d}W^{(t)}\vec{v}^*$ for some $\vec{v}^* = (\vec{u}^*)^\top W^*$ with $\vec{u}^* \in \R^p$.
Propagating the equations over the first two steps, and using Equation \eqref{eq:M_update}, we show that $M_{\vec{v}^*}^{(2)}$ can be expressed as an expectation w.r.t the pre-activations $\vec{h^\star}$ of a function dependent on the target $g^\star$, the second layer $\vec{a}$, and the activation function $\sigma$:
\begin{equation}\label{eq:overlap_}
    M_{\vec{v}^*}^{(2)} \!=\! (1 \!-\! \eta\lambda) M_{\vec{v}^*}^{(1)} + \eta \alpha \Eb{\vec{h^\star}}{F_{\sigma, a}(g^\star(\vec{h^\star}))\vec{h^\star}^\top}\vec{u}^\star\,.
\end{equation} 
The function $F_{\sigma, a}$ is described in App.\ref{app:sec:main_thm_proof},  Eq.(\ref{eq:F_def}).  Finally, we show an equivalence between the condition $\Eb{\vec{h^\star}}{F_{\sigma, a}(g^\star(\vec{h^\star}))\vec{h^\star}^\top}\vec{u}^\star = 0$ to the condition $\Eb{\vec{z}}{F(f^\star(\vec{z}))\langle \vec{v}^\star, \vec{z}\rangle}=0$ for general $F$ in definition \ref{def:two_step_hard}.

\paragraph{General multi-pass schemes ---}
\label{sec:main:extensions}
While Theorem \ref{thm:main:2step_learning} considers finite number of updates with the same batch of data for each step, it can be naturally generalized to other setups involving multiple-passes over a finite-number of mini-batches of size $\mathcal{O}(d)$. For instance, one can cycle over distinct minibatches with each cycle constituting one epoch or pass through the dataset. Theorem \ref{thm:main:2step_learning} remains valid under such a setup with the onset of weak-recovery shifting to the start of the second epoch instead of the second gradient step. We provide a sketch of this extension in Appendix \ref{sec:app:repeat}. On the other hand, if the minibatches are sampled with replacement from the dataset, the weak recovery still starts at the second gradient step. We illustrate this in Fig.\ref{fig:overlaps_batch} (in appendix). Furthermore, we empirically observe that the phenomenon holds even when considering the limit of mini-batch size $1$ (Figure \ref{fig:overlaps_minibatch} Appendix). Proving this,  however, 
remains out of the reach of the present technique.

%% file: sections/conclusions.tex
\section{Conclusions}
\label{sec:main:conclusion}
Our study analyzes the training dynamics of two-layer neural networks for learning multi-index target functions, distinctively focusing on {\it multi-pass} gradient descent which involves reusing batches multiple times. We find that this enables gradient descent to exceed the constraints imposed by information and leap exponents. 

Gradient descent is found to achieve a positive correlation with the target function across a broader class than previously anticipated, with only two data batch repetitions. Our analysis further demonstrates that the limitations associated with information and leap exponents, staircase learning, and CSQ lower bounds are restricted to online/single pass SGD and do not describe the class of functions inherently easy or hard to learn by gradient-based methods for neural networks. 

Our conclusions follow from rigorous mathematical proofs derived from Dynamical Mean Field Theory, through which we also offer an analytical description of the dynamic processes of low-dimensional weight projections—a noteworthy insight. Additionally, we provide a closed-form depiction of these dynamical processes and illustrate our theoretical findings with numerical experiments.

%% file: sections/appendix/proofs.tex
\section{Mathematical Proofs}\label{sec:app:proofs}

\subsection{Notations}
We use the asymptotic notation $f_1(d)=\Theta_d(f_1(d))$ to denote $c \abs{f_1(d)} \leq \abs{f_2(d)} \leq C \abs{f_1(d)}$ for some constants $c, C > 0$ and large enough $d$. Similarly, $f_1(d)=o_d(f_1(d))$ denotes
$f_1(d) \leq c(f_1(d))$ for any constant $c > 0$ and large enough $d$.
We use $\xrightarrow[P]{d,n \rightarrow \infty}, \xrightarrow[D]{d,n \rightarrow \infty}$ to denote convergence in probability and convergence in distribution respectively as $d,n \rightarrow \infty$ with $n/d=\alpha > 0$. We denote subspaces and linear operators, matrices on  $\R^d$ through uppercase letters $A,B,C,\cdots$. For any subspace $A \in \R^d$, we denote by $A_\perp$, its orthogonal complement, i.e the subspace of vectors orthogonal to all $\vec{v} \in A$.
\subsection{DMFT and iterative conditioning}\label{sec:iterative}
Unlike online SGD, the preactivations after multiple steps no longer remain Gaussian since the weights become dependent on the data. This prevents marginalizing over the orthogonal components over the preactivations and relating the learning of new directions to the Hermite decomposition of the target function.
Our proof circumvents these issues by utilizing a simpler effective process that decouples the pre-activations for different samples. The effective process is obtained using a rigorous version of the Dynamical Mean Field Theory derived in \citep{Montanari2022} and \citep{gerbelot2023rigorous}.

The derivation of Dynamical Mean Field Theory in the above works has the following essential elements:

\begin{itemize}[noitemsep,leftmargin=1em,wide=0pt]
    \item[1.] Iterative conditioning: The proof in  \citep{gerbelot2023rigorous,Montanari2022} for obtaining the DMFT equations relies on the observation that the gradient descent algorithm in Equation \eqref{eq:act_update} for a finite-number of iterations can be described completely through projections of the inputs design matrix $\vec{Z} \in \R^{n \times d}$ along a finite number of vectors in $\R^n, \R^d$. The iterative conditioning technique \citep{bolthausen2014iterative,bayati2011dynamics} then involves replacing the components of $\vec{Z} \in \R^{n \times d}$ along directions orthogonal to these projections by independent Gaussian random variables. This leads to a non-Markovian structure in the effective processes for the activations, parameters. 
     \item[2.] The concentration of finite-dimensional order parameters such as overlaps of the neuron parameters with the teacher neurons/subspace as well as expectations w.r.t the empirical measure of the pre-activations and parameters. 
\end{itemize}

Using the above elements, DMFT provides a low-dimensional effective dynamics characterizing the limiting joint empirical measure of the student parameters, as well as the pre-activations. 
We illustrate the proof for the activations after the first gradient step, illustrating the relationship with the ``hidden progress" described in section \ref{sec:hidden_prog}

Let $\vec{H}^{(t)} = \frac{1}{\sqrt{d}}\vec{Z}(\vec{W}^{(t)})^\top$ denote the $n \times p$ matrix of pre-activations at time $t$. Similarly, let $\vec{H}^* \in \R^{n\times k}$ denote the matrix of input activations in the target function. We denote by $\nabla_{\vec{H}} \mathcal{L}(\vec{H}^{*},\vec{H}^{(t)}) \in \R^{n \times p}$, the matrix derivative of $\mathcal{L}$ w.r.t the corresponding entries of the pre-activations matrix $\vec{H}^{(t)}$.
After each gradient update, $\vec{W}^{(t)}$ and the preactivations $\vec{H}^{(t)}$ gain a dependence on $\vec{Z}$.
 The Iterative conditioning technique works around this dependence by 
 conditioning on the sigma algebra generated by $\vec{H}^{(t)},\vec{H}^*$ and $\vec{W}^{(t)}$ instead of on $\vec{Z}$.
 Since $\vec{Z}$ interacts with $\vec{W}^{(t)}$ and 
$\vec{H}^{(t)}$ only through projections (along right with $\vec{W}^{(t)}$ and left with $\nabla_{\vec{H}} \mathcal{L}(\vec{H}^{*},\vec{H}^{(t)})$ respectively), the conditioning allows the components of $\vec{Z}$ orthogonal to $\nabla_H \mathcal{L}(\vec{H}^{(*)},\vec{H}^{(t)})$ and $\vec{W}^{(t)}$ to be replaced by independent Gaussian entries.  

For the first-gradient step, we only require conditioning on $\vec{H}^{(0)}, \vec{H}^{*}, \vec{W}^{0}$.

From (\eqref{eq:GD_update}), we obtain the following update for $\vec{H}^{(1)}$:
\begin{equation}\label{eq:act_update}
    \vec{H}^{(1)}= \vec{H}^{(0)} + \frac{\eta}{d}\vec{Z}(\vec{Z})^\top  \nabla_{\vec{H}} \mathcal{L}(\vec{H}^*, \vec{H}^{(0)})\vec{a}^\top,
\end{equation}
Let $\vec{\bar{W}}^{(0)} = \begin{pmatrix}
    \vec{W}^{(0)}\\
    \vec{W}^{\star}
\end{pmatrix} \in \R^{p+k \times d}$
By the equivalence of projection and conditioning for Gaussian random variables, we have that the following inequality holds in distribution:
\begin{equation}
\vec{Z}\vert_{\vec{H}^{0},\vec{H}^*,\vec{\bar{W}}^{(0)}} \overset{d}{=} \vec{Z}P_w^\top + \tilde{\vec{Z}}(P^{\perp}_w)^\top,
\end{equation}
where $\tilde{\vec{Z}}$ is independent of $\vec{Z}$ and $P_w$ denote the projection operator along $\vec{\bar{W}}^{(0)},\vec{\bar{W}}^{\star}$, defined as:
\begin{align*}
    P_w = \vec{\bar{W}}^{(0)} (\vec{\bar{W}}^{(0)}(\vec{\bar{W}}^{(0)})^\top)^{-1}(\vec{\bar{W}}^{(0)})^\top. 
\end{align*}

Substituting in Equation (\eqref{eq:act_update}), we obtain:
\begin{equation}\label{eq:act_iterative}
    \vec{H}^{(1)} \overset{d}{=} \vec{H}^{(0)} +
    \eta \frac{1}{d}\tilde{\vec{Z}}(P^{\perp}_w)^\top P^{\perp}_w(\tilde{\vec{Z}})^\top \nabla_{\vec{H}} \mathcal{L}(\vec{H}^*, \vec{H}^{(0)})\vec{a}^\top + \eta\frac{1}{d}\vec{H_0}(\vec{\bar{W}}^{(0)}(\vec{\bar{W}}^{(0)})^\top)^{-1}(\vec{H_0})^\top \nabla_{\vec{H}} \mathcal{L}(\vec{H}^*, \vec{H}^{(0)})\vec{a}^\top
\end{equation}
Since the projection, $P_w$ is along a low-dimensional subspace of dimension at most $p$, we have $P^{\perp}_w \approx \mathbf{I}_d$. 
One can therefore show that $\frac{1}{\sqrt{d}}\tilde{\vec{Z}}(P^{\perp}_w)^\top P^{\perp}_w(\tilde{\vec{Z}})^\top \vec{u}$ converges in probability to $\tilde{\vec{Z}} \tilde{\vec{Z}}^\top \vec{u}$. for any deterministic $\vec{u} \in \R^d$ with $\norm{u}= \cO(\sqrt{d})$. Applying it to the vector $\vec{u} = \mathcal{L}(\vec{H}^*, \vec{H}^{(0)})$, conditioned on $\vec{H}^*, \vec{H}^{(0)}$, we obtain that:
\begin{equation}
 \frac{1}{\sqrt{d}}\norm{\frac{1}{d}\tilde{\vec{Z}}(P^{\perp}_w)^\top P^{\perp}_w(\tilde{\vec{Z}})^\top \nabla_{\vec{H}} \mathcal{L}(\vec{H}^*, \vec{H}^{(0)}) - \frac{1}{d}\tilde{\vec{Z}}(\tilde{\vec{Z}})^\top \nabla_{\vec{H}}\mathcal{L}(\vec{H}^*, \vec{H}^{(0)})}_F    \xrightarrow[P]{n,d \rightarrow \infty} 0. 
\end{equation}
Now, the diagonal entries of $\frac{1}{d}\tilde{\vec{Z}}(\tilde{\vec{Z}})^\top$ convergence in probability to $1$ due to the concentration of norms of Gaussian random vectors. This results in the term $\mathcal{L}(\vec{H}^*, \vec{H}^{(0)})$. This term is precisely the one responsible for the ``hidden progress" explained in section \ref{sec:hidden_prog}, corresponding to the term in Equation \ref{eq:hid_prog_term}. Since $\tilde{\vec{Z}}$ is independent of $\vec{W}^{(0)}$ and $\vec{H}^{(0)}, \vec{H}^*$, by central limit-theorem for sub-Gaussian random variables, the remaining off-diagonal terms can be shown to converge to Gaussian noise independent of $\vec{H}^{(0)}, \vec{H}^*$ with variance $\norm{\nabla_{\vec{h}} \mathcal{L}(\vec{h}^*, \vec{h}^{(0)})}^2$. 

Lastly, the third term in Equation \eqref{eq:act_iterative} can be shown to converge to Gaussian noise correlated with corresponding entries of $\vec{H_0},\vec{H^*}$.
Specifically, by removing the conditioning on $\vec{H_0}$, we have through law of large numbers and Stein's Lemma, we have that the term $\frac{1}{d}(\vec{H_0})^\top \nabla_{\vec{H}} \mathcal{L}(\vec{H}^*, \vec{H}^{(0)})$ converges in probability to $(\vec{W}^{(0)}(\vec{W}^{(0)})^\top) \Ea{\nabla^2_{\vec{h}}\mathcal{L}(\vec{h}^*, \vec{h}^{(0)})}$. Therefore, we obtain:
\begin{equation}
    \frac{1}{\sqrt{d}}\norm{\frac{1}{d}\vec{H_0}(\vec{W}^{(0)}(\vec{W}^{(0)})^\top)^{-1}(\vec{H_0})^\top \nabla_{\vec{H}} \mathcal{L}(\vec{H}^*, \vec{h}^{(0)})\vec{a}^\top - \vec{H_0}\Ea{\nabla^2_{\vec{h}}\mathcal{L}(\vec{h}^*, \vec{h}^{(0)})})\vec{a}^\top} \xrightarrow[P]{n,d \rightarrow \infty} 0
\end{equation}

Proceeding similarly, one obtains low-dimensional effective processes for $\vec{W}^{(t)},\vec{H}^{(t)}$ for any time $t \in \N$.  In the following section, we derive the resulting DMFT dynamics for the setup considered in Section \ref{sec:main:introduction} through a reduction to the result in \citep{gerbelot2023rigorous}. We refer to \citep{gerbelot2023rigorous,celentano2021highdimensional} for detailed proofs based on the above technique.




\subsection{Derivation of the exact asymptotics}

We start by stating a general consequence of the main result in \citep{gerbelot2023rigorous}.

\begin{theorem}[Corollary of Theorem 3.2 in \citep{gerbelot2023rigorous}]
\label{thm:Cedric_theorem}
Let $W^{0} \in \R^{q \times d}$ be a sequence of matrices such that the overlap matrix satifies:
\begin{equation}
    \frac{1}{d}W^{0}(W^{0})^\top \xrightarrow[a.s]{d \rightarrow \infty} Q^{(0)},
\end{equation}
where $Q^{(0)} \in \mathcal{S}^+_p$ denotes
a fixed matrix.  Consider a dynamics of the form:
    \begin{align}
        W^{(t+1)} =  W^{(t)} - \eta \lambda W^{(t)} - \eta \frac{1}{\sqrt{d}}\sum_{\nu=1}^n F\left(\frac{ W^{(t)}\vec z_\nu} {\sqrt{d}}\right) \vec z_\nu^\top
    \end{align}
where $F:\R^q \rightarrow \R^q$ is pseudo-Lipshitz of finite-order and $\{z_\nu\}_{\nu =1}^n $ are i.i.d vectors distributed as $z_\nu \sim \mathcal{N}(0, \mathbb{I}_d)$, such that $n,d \rightarrow \infty$ with $n/d =\alpha >0$.
    Then the empirical measure of the weights $ \vec w^{(t)}_i$ converges in distribution to the weight process $\vec \theta_i^{(t)}$ and the  empirical measure of the preactivations $\vec h^{(t)}_\nu$ converges in distribution to that of the  preactivation process $\vec h^{(t)}$, defined as    
    \begin{equation}\label{eq:proc_theta}
        \boldsymbol{\theta}^{(t+1)} -  \boldsymbol{\theta}^{(t)}= - \,\eta\, \left(\lambda + \Lambda^{(t)}\right)\boldsymbol{\theta}^{(t)} +\,\eta \sum_{\tau=0}^{t-1}  R_\ell^{(t,\tau)} \boldsymbol{\theta}^{(\tau)} + \,\eta \boldsymbol{u}^{(t)}
    \end{equation}
    \begin{equation}\label{eq:proc_h}
        \boldsymbol{h}^{(t)} = -\eta\sum_{\tau=0}^{t-1} R_\theta^{(t,\tau)} F(\boldsymbol{h}^{(\tau)}) + \boldsymbol{\omega}^{(t)}
    \end{equation}
    were we have
    \begin{equation}
        \Lambda^{(t)} = \alpha\mathbb{E}\left[ \nabla_{\boldsymbol{h}} F (\boldsymbol{h}^{(t)})\right]
    \end{equation}
    \begin{align} 
        R_\theta^{(t,\tau)} = \mathbb{E}\left[\frac{\partial\, \vec \theta^{(t)}}{\partial\, \vec \tau^{(\tau)}} \right]
    \end{align}
    \begin{align}
        R_\ell^{(t,\tau)} = \alpha \mathbb{E}\left[\frac{\partial\, F(\vec h^{(t)})}{\partial\, \vec\omega^{(\tau)}} \right]\,.
    \end{align}
    Finally, $\boldsymbol{u}^{(t)}$ and $\boldsymbol{\omega}^{(t)}$ are zero-mean Gaussian processes respectively with covariances given by $C_\ell^{(t,\tau)}$ and $C_\theta^{(t,\tau)}$:
    
    \begin{equation}
        C_\ell^{(t, \tau)} = \alpha\mathbb{E}\left[ \boldsymbol{u}^{(t)}\left(\boldsymbol{u}^{(\tau)}\right)^\top \right] = \alpha \mathbb{E}\left[ F(\boldsymbol{r}^{(t)})F (\boldsymbol{r}^{(\tau)})^\top\right]
    \end{equation}
    \begin{equation}
        C_\theta^{(t, \tau)} = \mathbb{E}\left[\boldsymbol{\omega}^{(t)}\left(\boldsymbol{\omega}^{(\tau)} \right)^\top\right] = \mathbb{E}\left[ \boldsymbol{\theta}^{(t)}\left(\boldsymbol{\theta}^{(\tau)}\right)^\top \right]
    \end{equation}

The convergence in distribution of the empirical measures holds in the following sense:
    For any $t \in \mathbb{N}$, and any 
        pseudo-Lipschitz functions $\psi: \mathbb{R}^{p(t+1)} \to \mathbb{R}$ and $\phi: \mathbb{R}^{pt} \to \mathbb{R}$:
        \begin{align}
            &\frac{1}{d}\sum_{i=1}^{d}\psi((W_{i}^{(0)}, ..., W_{i}^{(t)})) \xrightarrow[n,d \to 
            \infty]{\whp} \mathbb{E}\left[\psi(\boldsymbol{\theta}^{(0)}, ...,\boldsymbol{\theta}^{(t)})\right], \\ &\frac{1}{n}\sum_{\nu=1}^{n}\phi((\mathbf{h}_\nu^{(0)}, ..., \mathbf{h}_{\nu}^{(t-1)})) \xrightarrow[n,d \to \infty]{\whp} \mathbb{E}\left[\phi(\vec{h}^{(0)}, ..., \vec{h}^{(t-1)})\right],
        \end{align}
where $W_{i}^{(t)}$ denotes the $i_{th}$ column of $W^{(t)}$
\end{theorem}
The above result follows directly by substituting $\frac{1}{\sqrt{d}} \vec{Z}$ as $\vec{X}$ in Theorem 3.2 of \cite{gerbelot2023rigorous}.

The definitions of $R_\theta$ and $R_\ell$ in Theorem \ref{thm:Cedric_theorem} require differentiating through the non-markovian processes defined by Equations \ref{eq:proc_h}, \ref{eq:proc_theta}. Fortunately,  $R_\theta, R_\ell$ can be equivalently described through an explicit set of recursive updates, which we state below for convenience:
    \begin{equation}
        R_\theta^{(t+1, \tau)} - R_\theta^{(t, \tau)} = - \,\eta\, \left(\lambda + \Lambda^{(t)}\right)R_\theta^{(t, \tau)} +\,\eta \sum_{s=\tau}^{t-1}  R_\ell^{(t,s)} R_\theta^{(\tau, s)} 
    \end{equation}
    with boundary conditions
    \begin{align}
        &R_\theta^{(t, t)} = 1\,, \\
        &R_\theta^{(t+1, t)} = 1 - \,\eta\, \Lambda^{(t)}\,,
    \end{align}
while 
\begin{equation}
    R_\ell^{(t, \tau)} = \alpha\mathbb{E}\left[\nabla_{\boldsymbol{h}} F (\boldsymbol{h}^{(t)})\, T_\ell^{(t, \tau)} \right],
\end{equation}
where $T_\ell^{(t, \tau)}$ is a collection of stochastic processes with distribution 
\begin{equation}
    T_\ell^{(t, \tau)} = R_\theta^{(t,\tau)}\nabla_{\boldsymbol{h}} F (\boldsymbol{h}^{(\tau)}) + \sum_{s=\tau+1}^{t-1} R_\theta^{(t,s)}  T_\ell^{(s, \tau)} 
\end{equation}
and boundary conditions
\begin{align}
    &T_\ell^{(t, t)} = 0\,, \\
    &T_\ell^{(t+1, t)} = 1 - \,\eta\, \Lambda^{(t)}\,,
\end{align}

To obtain the limiting equations under the setting of gradient descent with teacher weights $W^*$ in section \ref{sec:main:introduction}, we utilize the generality of the update $F$ in theorem \ref{thm:Cedric_theorem}, which allows for a portion of the parameters ($W^*$) to remain unaffected. We obtain the following result, which generalize the former theorem to the setting of our paper:
\begin{theorem}\label{thm:DMFT_committee}
Consider the distribution over data defined in section \ref{sec:main:setting} and an update rule on the weights of the form \eqref{eq:GD_update}, i.e:
    \begin{align} 
    \vec w_i^{(t+1)} = \vec w^{(t)}_i - \eta \lambda \vec w^{(t)}_i - \eta \sum_{\nu=1}^n\nabla_{\vec w^{(t)}_i} \,\mathcal{L}\left(\frac{W^{(t)}\vec z_\nu} {\sqrt{d}}, \frac{W^\star\vec z_\nu} {\sqrt{d}} \right)\,,
    \end{align}
Then under the assumptions of Theorem \ref{thm:main:2step_learning}, as $d \rightarrow \infty$ with $n/d=\alpha >0$, the joint empirical measure of the coordinates of the student weights $\vec w_i^{(t)}$ and the teacher weights
$W^\star$  converges in distribution to  the stochastic process $\vec{\theta}^{(t)}$ and the standard normal variable $\vec{\theta^\star}$, in the sense of Theorem \ref{thm:Cedric_theorem}. Similarly, the joint empirical measure of the student and teacher preactivations $\frac{\vec w_i^{(t)}\vec z_\nu} {\sqrt{d}}$, $\frac{\vec w_i^\star\vec z_\nu} {\sqrt{d}}$ converge in distribution to the stochastic process
$\vec{h}^{(t)}$ and the standard normal variable $\vec{h}^\star$. $\boldsymbol{\theta}^{(t)}$ and $\vec{h}^{(t)}$ are defined recursively through the following equations:
\begin{equation} \label{eq:def_theta_planted}
    \boldsymbol{\theta}^{(t+1)} -  \boldsymbol{\theta}^{(t)}= - \,\eta\,\left(\lambda + \Lambda^{(t)}\right)\boldsymbol{\theta}^{(t)} +\,\eta \sum_{\tau=0}^{t-1}  R_\ell^{(t,\tau)}\boldsymbol{\theta}^{(\tau)} - \,\eta g^{(t)}\boldsymbol{\theta}^\star + \,\eta \sum_{\tau=0}^{t}  \tilde R_\ell^{(t,\tau)}\boldsymbol{\theta}^\star + \,\eta \boldsymbol{u}^{(t)}
\end{equation}
\begin{equation} \label{eq:def_h_planted}
    \boldsymbol{h}^{(t)} = -\eta\sum_{\tau=0}^{t-1}R_\theta^{(t,\tau)} \nabla_{\boldsymbol{h}} \ell(\boldsymbol{h}^{(\tau)}, \boldsymbol{h}^\star) + \boldsymbol{\omega}^{(t)}
\end{equation}
Here $\vec u^{(t)}, \boldsymbol \theta^\star$ and $(\boldsymbol \omega^{(t)},  \boldsymbol h^\star)$ are zero mean Gaussian Process with covariances $C_\ell^{(t,\tau)}$ and $\Omega^{(t,\tau)}$ respectively, given by:
\begin{align} 
    C_\ell^{(t, \tau)} \!\!=   \alpha\mathbb{E}\left[ \boldsymbol{u}^{(t)}\left(\boldsymbol{u}^{(\tau)}\right)^\top \right] =
    \alpha \mathbb{E}_{\boldsymbol{h}^{(t)}, \boldsymbol{h}^\star}\!\left[ \nabla_{\boldsymbol{h}} \mathcal{L}(\boldsymbol{h}^{(t)}, \boldsymbol{h}^\star)\nabla_{\boldsymbol{h}} \mathcal{L}(\boldsymbol{h}^{(\tau)}, \boldsymbol{h}^\star)^\top\!\right] \label{eq:def_C_l} \\
    \Omega^{(t,\tau)} \!\!= \!\!  \mathbb{E}\!\left[ \begin{pmatrix}
       \boldsymbol \omega^{(t)} \\ \!\!  \boldsymbol h^\star
    \end{pmatrix} \begin{pmatrix}
        \boldsymbol \omega^{(t)} \!  \\ \boldsymbol h^\star \!
    \end{pmatrix}^\top \!\right]= \begin{bmatrix}
        C_\theta^{(t,\tau)} & M^{(t)} \\
        M^{(t)} & 1
    \end{bmatrix}\label{eq:def_C_omega},
\end{align}
where $C_\theta^{(t,\tau)}, M^{(t)}$ are defined as:
\begin{align}
\begin{bmatrix}
        C_\theta^{(t,\tau)} & M^{(t)} \\
        M^{(t)} & 1
    \end{bmatrix} \!\!=\!  \mathbb{E}_{\boldsymbol{\theta}^{(t)}, \boldsymbol{\theta}^\star}\!\left[ \begin{pmatrix}
        \boldsymbol{\theta}^{(t)} \\ \!\! \boldsymbol{\theta}^\star
    \end{pmatrix} \begin{pmatrix}
        \boldsymbol{\theta}^{(\tau)} \!  \\ \boldsymbol{\theta}^\star \!
    \end{pmatrix}^\top \!\right]
\end{align}
The effective regularisation $\Lambda^{(t)}$ and the projected gradient $g^{(t)}$ concentrate to
\begin{align}
    \Lambda^{(t)} \!=\! \alpha\mathbb{E}_{\boldsymbol{h}^{(t)}, \boldsymbol{h}^\star}\left[\nabla^2_{\vec h}\mathcal{L}\left( \vec h^{(t)}, \boldsymbol{h}^\star\right)\right]\,,\,\\ g^{(t)}\!=\!\alpha \mathbb{E}_{\boldsymbol{h}^{(t)}, \boldsymbol{h}^\star}\left[\nabla_{\vec h}\mathcal{L}\left( \vec h^{(t)}, \boldsymbol{h}^\star\right) \vec h^{\star\top}\right]
\end{align}
The memory kernels $R_\ell^{(t,\tau)}$, $\tilde R_\ell^{(t,\tau)}$, $R_\theta^{(t,\tau)}$ for $t > \tau$ are defined through the partial derivatives with respect to the noise:
\begin{eqnarray}
    &R_\theta^{(t,\tau)} = \mathbb{E}_{\boldsymbol{\theta}^{(t)}, \boldsymbol{\theta}^\star}\left[\frac{\partial\, \vec \theta^{(t)}}{\partial\, \vec \tau^{(\tau)}} \right]\,, \nonumber  \\
    &R_\ell^{(t,\tau)} = \alpha \mathbb{E}_{\boldsymbol{h}^{(t)}, \boldsymbol{h}^\star}\left[\frac{\partial\, \nabla_{\vec h}\mathcal{L}\left( \vec h^{(t)}, \boldsymbol{h}^\star\right)}{\partial\, \boldsymbol\omega^{(\tau)}} \right]\,,\\
    &\tilde R_\ell^{(t,\tau)} = \alpha \mathbb{E}_{\boldsymbol{h}^{(t)}, \boldsymbol{h}^\star}\left[\frac{\partial\, \nabla_{\vec h}\mathcal{L}\left( \vec h^{(t)}, \boldsymbol{h}^\star\right)}{\partial\, \left(\boldsymbol\omega^\star\right)^{(\tau)}} \right]\,,\quad
\end{eqnarray}
and $R_\ell^{(t,t)} = \tilde R_\ell^{(t,t)} = 0$, $R_\theta^{(t,t)}=1$.
Finally, $M^{(t)}$ satisfies the update equation    
\begin{equation} \label{eq:def_M_rec}
    M^{(t+1)} =(1-\eta\lambda) M^{(t)} -\eta g^{(t)}\,,
\end{equation}
where $g^{(t)}$ is defined as:
\begin{equation}\label{eq:def_gt}
    \alpha \mathbb{E}\left[ \nabla_{\boldsymbol{h}} \ell(\boldsymbol{h}^{(t)}) \left( \boldsymbol{h}^* \right)^\top\right]
\end{equation}
\end{theorem}

\begin{proof}
  Analogous to the embedding of planted vectors in \citep{celentano2021highdimensional}, we start by considering an a lifted dynamics defined by concating $W^{(t)}$ and $W^*$. First, define the extended parameters ${\tilde W}^{(t)}\in \mathbb{R}^{(p+k)\times d}$ with update rule:
    
    \begin{align} \label{eq:combined_GD}
        \begin{bmatrix} W^{(t+1)} \\ W^* \end{bmatrix} = \begin{bmatrix} W^{(t)} \\ W^* \end{bmatrix} -\eta \begin{bmatrix} \lambda W^{(t)} \\ 0 \end{bmatrix} - \eta \sum_{\nu=1}^n\begin{bmatrix}
            \nabla_{W^{(t)}} \,\mathcal{L}\left(\frac{W^{(t)}\vec z^\nu} {\sqrt{d}}\right) \\0
        \end{bmatrix}\,,
    \end{align}
   The above form of updates can be seen to be a special case of Theorem \ref{thm:Cedric_theorem} with $q=p+k$ and $F: R^{p+k} \rightarrow \R^{p+k}$ given by:
   \begin{equation}
       F: \begin{pmatrix}
           \vec{h}\\
           \vec{h}^\star
       \end{pmatrix} \rightarrow 
       \begin{pmatrix}\nabla_{\vec{h}}\mathcal{L}(\vec{h},\vec{h}^\star) \\
       \vec{0}
        \end{pmatrix}
   \end{equation}
The assumptions on $g^*, \sigma$ imply that $F$ is pseudo-Lipschitz of finite-order while standard concentration results for sub-exponential random variables when applied to $W^{(0)},W^{*}$ imply that the overlap matrices at initialization converge almost surely. Therefore, Theorem \ref{thm:Cedric_theorem} applies, with the effective process for the weights and the pre-activations being described by: 
    \begin{equation} \label{eq:combined_theta}
        \begin{bmatrix} \boldsymbol{\theta}^{(t+1)} \\ \boldsymbol{\theta}^* \end{bmatrix} -  \begin{bmatrix} \boldsymbol{\theta}^{(t)} \\ \boldsymbol{\theta}^* \end{bmatrix}= - \,\eta\, \begin{bmatrix}\lambda + \Lambda^{(t)} & \tilde\Lambda^{(t)}\\0 & 0\end{bmatrix}\begin{bmatrix} \boldsymbol{\theta}^{(t)} \\ \boldsymbol{\theta}^* \end{bmatrix} +\,\eta \sum_{\tau=0}^{t}  \begin{bmatrix} R_\ell^{(t,\tau)} & \tilde{R}_\ell^{(t,\tau)} \\ 0 & 0 \end{bmatrix}\begin{bmatrix} \boldsymbol{\theta}^{(\tau)} \\ \boldsymbol{\theta}^* \end{bmatrix} + \,\eta \begin{bmatrix} \boldsymbol{u}^{(t)} \\ 0 \end{bmatrix}
    \end{equation}
    \begin{equation} \label{eq:combined_h}
        \begin{bmatrix} \boldsymbol{h}^{(t)} \\ \boldsymbol{h}^* \end{bmatrix} = -\eta\sum_{\tau=0}^{t-1} \begin{bmatrix}
            R_\theta^{(t,\tau)} & \tilde{R}_\theta^{(t,\tau)} \\ 0 & 1
        \end{bmatrix} \begin{bmatrix} \nabla_{\boldsymbol{h}} \mathcal{L}(\boldsymbol{h}^{(\tau)}, \boldsymbol{h}^\star) \\ 0 \end{bmatrix} + \begin{bmatrix} \boldsymbol{\omega}^{(t)} \\ \boldsymbol{\omega}^* \end{bmatrix}
    \end{equation}
    Notice the redundancy in the above equations due to $W^*$ not being updated in \eqref{eq:combined_GD}. This allows us to further simplify \eqref{eq:combined_h} and \eqref{eq:combined_theta}, obtaining:
    \begin{equation} \label{eq:planted_theta_implicit}
        \boldsymbol{\theta}^{(t+1)} -  \boldsymbol{\theta}^{(t)}= - \,\eta\,\left(\lambda + \Lambda^{(t)}\right)\boldsymbol{\theta}^{(t)} +\,\eta \sum_{\tau=0}^{t-1}  R_\ell^{(t,\tau)}\boldsymbol{\theta}^{(\tau)} - \,\eta\,\tilde\Lambda^{(t)}\boldsymbol{\theta}^* +\,\eta \sum_{\tau=0}^{t}  \tilde R_\ell^{(t,\tau)}\boldsymbol{\theta}^* + \,\eta \boldsymbol{u}^{(t)}
    \end{equation}
    \begin{equation}
        \boldsymbol{h}^{(t)} = -\eta\sum_{\tau=0}^{t-1}R_\theta^{(t,\tau)} \nabla_{\boldsymbol{h}} \mathcal{L}(\boldsymbol{h}^{(\tau)}, \boldsymbol{h}^\star) + \boldsymbol{\omega}^{(t)}
    \end{equation}
    where we noticed that $\boldsymbol{h}^* \sim \boldsymbol{\omega}^*$. 
    These equations are the same as in \ref{thm:Cedric_theorem}, with just two extra terms in \eqref{eq:planted_theta_implicit}, $
    \tilde\Lambda^{(t)}$ and $\tilde R_\ell^{(t,\tau)}$. 
   An application of the Stein's Lemma further simplifies the term $\tilde\Lambda^{(t)}$ to $g^{(t)}$ in the Theorem as follows:
    \begin{equation}
        \tilde\Lambda^{(t)} = \alpha\mathbb{E}\left[ \nabla_{\boldsymbol{h}^*}\nabla_{\boldsymbol{h}} \ell(\boldsymbol{h}^{(t)}, \boldsymbol{h}^\star)\right] = \alpha \mathbb{E}\left[ \nabla_{\boldsymbol{h}} \ell(\boldsymbol{h}^{(t)}, \boldsymbol{h}^\star) \left( \boldsymbol{h}^* \right)^\top\right] = g^{(t)}, 
    \end{equation}

\end{proof}

The above effective process characterizes the limits of several quantities determined by the weights and pre-activations. In particular, it provides the limits of the student-teacher overlaps:
\begin{corollary}\label{cor:overlap}
Under the assumptions of Theorem \ref{thm:main:2step_learning}, 

\begin{equation}
    W^{(t)} (W^\star)^\top / d \xrightarrow[P]{n,d \rightarrow \infty} M^{(t)},
\end{equation}
where $M^{(t)}$ is defined as in Theorem \ref{thm:DMFT_committee}.
\end{corollary}
\begin{proof}
    Observe that $\langle \vec{w}^{(t)}_i, \vec{w}^\star \rangle / d$ can be expressed as an expectation of a pseudo-lipschitz function w.r.t the joint empirical measure over the coordinates of $\vec{w}^{(t)}_i,\vec{w}^\star$ with the value at the $j_{\rm th}$ coordinate given by $\{\vec{w}^{(t)}_i\}_j\{\vec{w}^\star\}_j$. Therefore $\ref{thm:main:2step_learning}$ implies that $W^{(t)} (W^\star)^\top / d$ converges in probability to the expected overlaps of the effective process $\vec \theta^{(t)},\vec{\theta}^\star$ which equal $M^{(t)}$ by definition.
\end{proof}
We also include a useful corollary, describing the evolution of the overlaps of the weights.

\begin{lemma}\label{lem:cov}
Under the assumptions of A.2 the covariance $C_\theta^{(t,\tau)}$
\begin{align}
    C_\theta^{(t+1,\tau)} - C_\theta^{(t,\tau)} = - \eta \left(\lambda + \Lambda^{(t)}\right)C_\theta^{(t,\tau)} + \eta \sum_{s=0}^{t-1}R_\ell^{(t,s)}C_\theta^{(s,\tau)} + \eta \sum_{s=0}^{\tau-1} R_\theta^{(t,s)}C_\ell^{(s,\tau)} - \\
    - \eta \left( g^{t} - \sum_{s=0}^{t}\tilde R_\ell^{(t,s)}\right) \left(M^{(\tau)}\right)^\top
\end{align}
\end{lemma}
This is a consequence of linearity of expectation on \eqref{eq:def_theta_planted}. Concretely, viewing $\theta^{(t)}$ as a function of the Gaussian random variables $\{u^{(\tau)}\}_{\tau=1}^t$, we apply the multi-variate Stein's Lemma to obtain:
\begin{equation}
    \mathbb{E}\left[ \vec\theta^{(t)} \vec u^{(\tau)}\right] = \sum_{s=\tau}^{t-1} R_\theta^{(t,s)}C_\ell^{(s,\tau)}\,.
\end{equation}

In particular, we obtain the following expression for the covariances upto the first time-steps:
\begin{lemma}\label{lem:cov_1}
The covariances $C_\theta^{(0,1)},C_\theta^{(0,0)}$ satisfy:
\begin{align}
    C_\theta^{(0,1)} - C_\theta^{(0,0)} =  - \eta \left(\lambda + \Lambda^{(0)}\right)C_\theta^{(0,0)} - \eta \left(g^{(0)} - \Lambda^{(0)}M^{(0)}\right) \left(M^{(0)}\right)^\top
\end{align}
\begin{align}
    C_\theta^{(1,1)} - C_\theta^{(0,1)} =  - \eta \left(\lambda + \Lambda^{(0)}\right)C_\theta^{(0,1)} + \eta C_\ell^{(0,1)} - \eta \left(g^{(0)} - \Lambda^{(0)}M^{(0)}\right) \left(M^{(1)}\right)^\top
\end{align}
\end{lemma}

\subsection{Pre-activations at the end of the first gradient update}\label{sec:app:proof_hidden_prog}
For $T=1$, Equation \eqref{eq:def_h_planted} simplifies to:

\begin{equation} \label{eq:h_1}
        \boldsymbol{h}^{(1)} = - \eta \nabla_{\boldsymbol{h}} \ell(\boldsymbol{h}^{(0)}, \vec{h}^\star) + \boldsymbol{\omega}^{(1)}
    \end{equation}

We now show that the first term exactly correspond to the contributions considered in section \ref{sec:hidden_prog}.

\begin{lemma}
Under the notation in section \ref{sec:hidden_prog} and assumptions of Theorem \ref{thm:main:2step_learning}:
\begin{equation}
      \frac{a}{d} \ell'\left(\vec h^{(0)}_\nu, \vec h^\star_\nu\right)\sigma'(\vec h^{(0)}_\nu)\langle \vec z_\nu, \vec z_\nu \rangle \xrightarrow[D]{n,d \rightarrow \infty} \nabla_{\boldsymbol{h}} \ell(\boldsymbol{h}^{0},\vec{h}^\star)
\end{equation}
where $\boldsymbol{h}^{0} \in \R^p, \boldsymbol{h}^* \in \R^k$ are independent Gaussian random variables distributed as in Theorem \ref{thm:DMFT_committee}.
\end{lemma}
\begin{proof}
    We simply apply the conditioning by projection technique described in Section \ref{sec:iterative} to $\vec z_\nu$ by expressing it as:
    $\vec{z}_\nu = \vec h^{(0)}_\nu + \frac{d-1}{d}\vec{z}'_\nu$, where $\vec{z}'_\nu$ is independent of $\vec h^{(0)}_\nu$.
    The result then follows from convergence in 
    probability of $\frac{1}{d} \langle \vec{z}'_\nu, \vec{z}'_\nu \rangle$ to $1$.
\end{proof}

Next, we characterize $\boldsymbol{\omega}^{(1)}$. We consider two cases:
\begin{itemize}
    \item $M^{(1)}=0$: In this case, Corollary \ref{cor:overlap} implies that the first-layer does not develop any overlap with directions in $U^*$. Equation \eqref{eq:def_C_omega} then implies that $\boldsymbol{\omega}^{(1)}$ is uncorrelated with $\omega^*$.
    \item  $M^{(1)}\neq 0$: In this case the first-layer develops an overlap along $U^*$. By initialization, we have that $M^{(0)}=0$.
    Equation \eqref{eq:def_gt} implies that $M^{(1)}$ is given by:
    \begin{equation} \label{eq:M_fist_step}
       M^{(1)}= - \eta \alpha \mathbb{E}\left[ \nabla_{\boldsymbol{h}} \ell(\boldsymbol{h}^{(0)}) \left( \boldsymbol{h}^* \right)^\top\right]
      \end{equation}

Due to the choice of symmetric initialization (Equation \ref{eq:init_sym}), we have $f(\boldsymbol{h}^{(0)}) = 0$. Therefore, $
\nabla_{\boldsymbol{h}} \mathcal{L}(\boldsymbol{h}^{(0)}) = -ag^*(\vec{h}^*)\sigma'(\boldsymbol{h}^{(0)})$. We thus obtain
\begin{equation}
       M^{(1)}= \eta \alpha \vec{a}\Ea{g^*(\vec{h}^*) \boldsymbol{h}^{(0)}(\boldsymbol{h}^*)^\top} =  \eta \alpha \vec{a}\odot\Ea{\boldsymbol{h}^{(0)}}\Ea{g^*(\vec{h}^*) (\boldsymbol{h}^*)^\top},
\end{equation}
where $\odot$ denotes element-wise multiplication and we used the independence of $\vec{h}^*, \boldsymbol{h}^{(0)}$.
Therefore, the rows of $M^{(1)}$ gains a rank-one spike along $\Ea{g^*(\vec{h}^*)(\boldsymbol{h}^*)}$. This matches the corresponding results for single $\cO(d)$ batch gradient steps under the online-setting \citep{ba2022high,dandi2023twolayer}.

By Equation \eqref{eq:def_C_omega}, $\boldsymbol{\omega}^{(1)}$ can be expressed as:
\begin{equation}\label{eq:overlap_staircase}
    \boldsymbol{\omega}^{(1)} = \boldsymbol{\omega_\perp}^{(1)}+ M^{(1)}\boldsymbol{h}^*\,,
\end{equation}
where $\omega_\perp$ is independent of $\boldsymbol{h}^*$.
\end{itemize}

\subsection{Proof of Theorem \ref{thm:main:2step_learning}}\label{app:sec:main_thm_proof}
To illustrate the learning of directions solely due to the hidden progress explained in section \ref{sec:hidden_prog}, we first focus on the case where $M^{(1)}=0$ i.e when the parameters develop no overlap along the target subspace in the first step.

From Corollary \ref{cor:overlap}, and Slutsky's theorem, we have that:
\begin{equation}\label{eq:overlap_1}
     \frac{1}{d} W^{2} (W^\star)^\top -  \frac{1}{d} W^{1} (W^\star)^\top\xrightarrow[P]{n,d \rightarrow \infty}- \eta \alpha \mathbb{E}\left[ \nabla_{\boldsymbol{h}} \mathcal{L}(\boldsymbol{h}^{(1)},\vec{h}^\star) \left( \boldsymbol{h}^\star \right)^\top\right],
\end{equation}
where from Equation \eqref{eq:h_1}, $\vec{h}^{(1)}$ can be expressed as a combination of $\nabla_{\boldsymbol{h}} \mathcal{L}(\boldsymbol{h}^{(1)},\vec{h}^\star)$ and a Gaussian random variable $\boldsymbol{\omega}^{(1)}$ independent of 
$\vec{h}^\star$. Furthermore, Lemma \ref{lem:cov} implies that the regularization strength $\lambda$ and step-size $\eta$ can be set such that the entries of $\boldsymbol{\omega}^{(1)}$ have unit-variance.
Now, suppose $\vec{v}^\star = (W^\star)^\top \vec{u}^\star$ for some fixed vector $\vec{u}^\star \in \R^p$. First, consider the case when  $\vec{v}^\star$ lies in the subspace $P^*_\perp$ as defined in definition \ref{def:two_step_hard}.

By projecting Equation \eqref{eq:overlap_1} along $\vec{u}^\star$, we obtain that:
\begin{equation}\label{eq:overlap_2}
     \frac{1}{d} W^{2}\vec{v}^\star  -  \frac{1}{d} W^{1}\vec{v}^\star \xrightarrow[P]{n,d \rightarrow \infty}- \eta \alpha \mathbb{E}\left[ \nabla_{\boldsymbol{h}} \mathcal{L}(\boldsymbol{h}^{(1)},\vec{h}^*) \left( \boldsymbol{h}^\star \right)^\top\vec{u}^\star\right],
\end{equation}
For squared loss, we have $-\nabla_{\boldsymbol{h}_j} \mathcal{L}(\boldsymbol{h}^{(t)},\vec{h}^\star) = a_j(g^\star(\vec{h}^\star)-f(\boldsymbol{h}^{(t)}))\sigma'(\boldsymbol{h}^{(t)}_j)$.

Therefore, the overlap for the $j_{th}$ neuron can be expressed as : 
\begin{equation}\label{eq:overlap}
    \frac{1}{d} \langle \vec{w}^{(2)}_i, \vec{v}^\star \rangle -  \frac{1}{d} \langle \vec{w}^{(1)}_i, \vec{v}^\star \rangle \xrightarrow[P]{n,d \rightarrow \infty} \eta \alpha
    \mathbb{E}\left[a_jg^\star(\vec{h}^\star)\sigma'(\boldsymbol{h}^{(1)}_j)\left( \boldsymbol{h}^\star \right)^\top\vec{u}^\star\right]-  \eta \alpha\mathbb{E}\left[a_jf(\boldsymbol{h}^{(1)})\sigma'(\boldsymbol{h}^{(1)}_j)\left( \boldsymbol{h}^\star \right)^\top\vec{u}^\star\right].
\end{equation}
We focus on the first term in the RHS.
By assumption, $\frac{1}{d} \langle \vec{w}^{(0)}_i, \vec{v}^\star \rangle ,\frac{1}{d} \langle \vec{w}^{(1)}_i, \vec{v}^\star \rangle $ converge in probability to $0$. Therefore, using Equation \eqref{eq:h_1}, we obtain:
\begin{align}
   \frac{1}{d} \langle \vec{w}^{(2)}_i, \vec{v}^* \rangle &\xrightarrow[P]{n,d \rightarrow \infty}\eta \alpha
    \mathbb{E}\left[a_jg^\star(\vec{h}^\star)\sigma'(\boldsymbol{h}^{(1)}_j)\left( \boldsymbol{h}^\star \right)^\top\vec{u}^\star\right]\\ &= \eta \alpha
    \mathbb{E}\left[a_jg^\star(\vec{h}^\star)\sigma'(- \nabla_{\boldsymbol{h}} \mathcal{L}(\boldsymbol{h}^{(0)}) + \boldsymbol{\omega}^{(1)})\left( \boldsymbol{h}^\star \right)^\top\vec{u}^\star\right].
\end{align}
Recall that by the choice of initialization, $C_\theta^{(0,0)}$ are diagonal with entries $1$ except for the off-diagonal entries corresponding to pairing of neurons through the symmetric initialization. Furthermore, by initialization, $M^{(0)} = \vec{0}$ and by assumption $M^{(1)} = \vec{0}$.
 
From Lemma \ref{lem:cov}, and the definitions of  we have that by setting $\eta\gamma=1$, 
the covariance $C_\theta^{(0,1)},C_\theta^{(1,1)}$ simplify to:
\begin{align}
    C_\theta^{(0,1)} =  - \eta  \Lambda^{(0)}C_\theta^{(0,0)}
\end{align}
\begin{align}
    C_\theta^{(1,1)} =   \Lambda^{(0)}C_\theta^{(0,1)} + \eta C_\ell^{(0,1)}
\end{align}
By definition, $\Lambda^{(0)},C_\ell^{(0,1)}$ have diagonal entries proportional to $a_j,a_j^2$ respectively. Therefore, we can further set $\eta > 0$ such that the $j_{th}$ diagonal entry of $C^{1,1}_{\theta}$ equals $1+a^2_j$. By case $1$ in section \ref{sec:app:proof_hidden_prog}, we further have that $\vec{\omega}^{(1)}$ is independent of $\vec{h}^*$.

Substituting $\nabla_{\boldsymbol{h}} \mathcal{L}(\boldsymbol{h}^{(0)}) = -ag^*(\vec{h}^*)\sigma'(\boldsymbol{h}^{(0)})$, we obtain the precise condition on $\sigma$ for the $j_{th}$ neuron to learn direction $\vec{v}^\star$ the second timestep. The condition is given by:
\begin{equation}\label{eq:phi}
   \phi(a_j)= \Ea{g^\star(\vec{h^\star})\sigma'(\eta a_jg^\star(\vec{h^\star})\sigma'(h^0_j)+a_j\xi)\langle\vec{h^\star},\vec{u}^\star\rangle} \neq 0,
\end{equation}
where $\vec{h^\star}$ and $\xi$ are independent Gaussian random variables. 
Since $\vec{h}^\star$ matches in distribution $\frac{1}{\sqrt{d}}\vec{W}^\star \vec{z}$, the above condition can equivalently be expressed as the following condition on $f^\star$:
\begin{equation}\label{eq:phi_def}
       \phi(a_j)= \Eb{\vec{z}}{F_{\sigma, \vec{a}}(f^\star(\vec{z}))\langle \vec{v}^\star, \vec{z}\rangle} \neq 0,
    \end{equation}
where:
\begin{equation}\label{eq:F_def}
    F_{\sigma, \vec{a}}(x) = \Eb{\xi_1,\xi_2}{x \sigma'(\eta a_j\sigma'(u)x  + a_j\xi)}, 
\end{equation}
where u is a standard normal variable, corresponding to $h_j^{(0)}$. The above expectation $\phi$ is an analytic function of $a_j$. To show that it is identically non-zero, we consider the derivative w.r.t $a_j$ at $a_j=0$. We have, using the dominated-convergence theorem: 
\begin{align*}
     \phi'(0) &= \eta \Ea{ \sigma'(u)g^\star(\vec{h^\star})^2\sigma''(0)\langle\vec{h^\star},\vec{u}^\star\rangle}\\
     &= \eta \Ea{ g^\star(\vec{h^\star})^2\langle\vec{h^\star},\vec{u}^\star\rangle}\Ea{\sigma'(u)}\sigma''(0)\\
     & = \eta \Ea{ g^\star(\vec{h^\star})^2\langle\vec{h^\star},\vec{u}^\star\rangle}\nu_{1}(\sigma)\sigma''(0),
\end{align*}
where $\nu_{1}(\sigma)$ denotes the $1_{st}$ Hermite-coefficients of $\sigma$. 
Similarly, iterating $k$-times, we have:
\begin{equation}
    D_{a_j}^{k}\phi(\vec{a}) = \eta^{k+1} \Ea{(g^\star(\vec{h^*}))^{k+1}\langle\vec{h^\star},\vec{u}^\star\rangle} \nu^k_{1}(\sigma)\sigma^{k+1}(0),
\end{equation}
where $\sigma^{k+1}(\sigma)$ denotes the $k_{th}$ derivative of $\sigma$. 
Note that for $\phi(\vec{a})$ to not be identically zero, it is sufficient that $D_{a_j}^{k-1}$ is non-zero for some $k \in \N$.
Since by assumption, $\vec{v}^* \in P^*_\perp$, and since the monomials $1,x,x^2,\cdots$ span the space of polynomials, we have that there exists a $k \in \N$ such that:
$\Ea{(g^\star(\vec{h^*}))^{k}\langle\vec{h^\star},\vec{u}^\star\rangle} \neq 0$. The conditions on $\sigma$ further imply that $\nu^k_{1}(\sigma)\sigma^{k+1}(0) \neq 0, \  \forall k \in \mathbb{N}$. Therefore $\phi(a_j)$ is a not identically $0$.



 Since $\phi(a_j)$ is an analytic function non-identically zero, and the law of $\vec a^{(0)} \sim \mathcal{N}(0,\frac{1}{p}{\mathbbm 1}_p)$ is absolutely continuous w.r.t the Lebesgue measure, we have that $\phi(a_j)$ is non-zero almost surely over the initialization. Now, the second term in Equation \ref{eq:overlap_2} is again an analytic function in $\vec{a}$, distinct from $\phi(a_j)$, and can therefore be almost surely absorbed into the non-zero overlap. This proves the first part of Theorem \ref{thm:main:2step_learning} for developing an overlap along a fixed direction in $P^*_\perp$ when $M^{(1)}=0$. We now proceed to show that the weights $W^{2}$ span $P^*_\perp$.

Let $r$ denote the dimension of the subspace $P^*_\perp$. Suppose that $v^*_1 = (W^\star)^\top \vec{u}_1^\star, v^*_2,\cdots = (W^\star)^\top \vec{u}_2^\star, \cdots, v^*_r = (W^\star)^\top \vec{u}_r^\star$ form an orthonormal basis of $P^*_\perp$. Let $V^* \in \R^{d \times r}$ matrix $M_u^* \in \R^{k \times r}$  denote matrices with columns$ \vec{v}_1^\star,\cdots,  \vec{v}_r^\star$ and $\vec{u}_1^\star,\cdots,  \vec{u}_r^\star$ respectively.

Analogous to Equation \eqref{eq:overlap}, we obtain:
\begin{equation}
     \frac{1}{d} W^{2}V^\star  -  \frac{1}{d} W^{1}V^\star \xrightarrow[P]{n,d \rightarrow \infty}- \eta \alpha \mathbb{E}\left[ \nabla_{\boldsymbol{h}} \mathcal{L}(\boldsymbol{h}^{(1)},\vec{h}^*) \left( \boldsymbol{h}^\star \right)^\top M_u^*\right],
\end{equation}

Following the derivation of Equation \eqref{eq:phi_def}, we obtain that the rows of the matrix $\mathbb{E}\left[\nabla_{\boldsymbol{h}} \mathcal{L}(\boldsymbol{h}^{(1)},\vec{h}^*) \left( \boldsymbol{h}^\star \right)^\top U^*\right]$ are independent for neurons $i,j$ for $j \neq p-i+1$ (due to the symmetric initialization in Equation \eqref{eq:init_sym}.
Furthermore each row of the matrix is absolutely continuous w.r.t the Lebesgue measure on $\R^r$. This implies that $\frac{1}{d} W^{2}V^\star  -  \frac{1}{d} W^{1}V^\star$ has full row-rank almost surely for large enough $p$.
 
Now, suppose that $\vec{v}^\star$ instead lies in the even-symmetric subspace $A^\star$. By induction and closure properties of analytic functions, we have that $\boldsymbol{h}^{(t)}$ can be expressed as:
\begin{equation}
   \boldsymbol{h}^{(t)} =  \mathcal{F}_t(\boldsymbol{h}^{(*)} , \omega_1,\boldsymbol{\omega}^{(1)},\cdots,\boldsymbol{\omega}^{(t)}),
\end{equation}
for an analytic mapping $\mathcal{F}_t$.
Now, similar to  Equation \eqref{eq:overlap}, we have that:
\begin{equation}
  \frac{1}{d} W^{t}\vec{v}^\star  -  \frac{1}{d} W^{t-1}\vec{v}^\star \xrightarrow[P]{n,d \rightarrow \infty}- \eta \alpha \mathbb{E}\left[ \nabla_{\boldsymbol{h}} \mathcal{L}(\boldsymbol{h}^{(t)},\vec{h}^*) \left( \boldsymbol{h}^\star \right)^\top\vec{u}^\star\right],
\end{equation}
Using Fubini's theorem, we may take expectation w.r.t to express each entry of $\mathbb{E}\left[ \nabla_{\boldsymbol{h}} \mathcal{L}(\boldsymbol{h}^{(t)},\vec{h}^*) \left( \boldsymbol{h}^\star \right)^\top\vec{u}^\star\right]$ as:
\begin{equation}
     \Eb{\vec{z}}{F_{t,\vec{a}}(f^\star(\vec{z}))\langle \vec{v}^\star, \vec{z}\rangle}\,,
\end{equation}
for some analytic $F_{t,\vec{a}}$
This ensures that the expectation in \eqref{eq:overlap_2} remains $0$ for all time $t$. This proves the second part of Theorem \ref{thm:main:2step_learning}.

\subsection{Effect of previously learned directions}\label{sec:app:staircase}

We now consider the case when $M^{(1)} \neq 0$, i.e when the first-layer develops an overlap along $U^*$. As shown in \ref{sec:hidden_prog}, the rows of $M^{(1)}$ lie along the same direction given by $\Ea{g^*(\vec{h}^*)(\boldsymbol{h}^*)}$. Without loss of generality, we assume that the direction $\Ea{g^*(\vec{h}^*)(\boldsymbol{h}^*)}$ corresponds to $\vec{e}_1$ in the input space $\R^d$ and that $W^*$ has rows along the standard basis $\vec{e}_1, \cdots, \vec{e}_k$. Note that $\vec{e}_1$ itself lies in $P^*_\perp$ by setting $F(x)=x$ in definition \ref{def:two_step_hard}.

From Equation \eqref{eq:overlap_staircase}, we obtain:
\begin{equation}
    \boldsymbol{\omega}_j^{(1)} = \boldsymbol{{\omega_{\perp}}_j}^{(1)}+\eta C a_j h^*_1\,
\end{equation}
where $C$ denotes a constant dependent on $g^*$. Since $\boldsymbol{\omega}^{(1)}$ is now correlated with $h^*_1$, the condition in Equation \eqref{eq:phi} is modified to:
\begin{equation}
   \phi(a_j)= \Ea{g^\star(\vec{h^\star})\sigma'(\eta a_j\sigma'(u)g^\star(\vec{h^\star})+\xi_1+a_j\xi_2+\eta C a_j h^*_1)\langle\vec{h^\star},\vec{u}^\star\rangle} \neq 0,
\end{equation}
Again, differentiating w.r.t $a_j$, we obtain:
\begin{equation}
    \phi'(0) = \eta \Ea{ g^\star(\vec{h^\star})^2\langle\vec{h^\star},\vec{u}^\star\rangle}\nu_1(\sigma)\nu_{2}(\sigma) + \eta \Ea{ g^\star(\vec{h^\star})h^*_1 \langle\vec{h^\star},\vec{u}^\star\rangle}
\end{equation}
Similar to section \ref{app:sec:main_thm_proof}, we have that $\vec{v}^* \in P^*_\perp$ is sufficient for the first term to be non-zero almost surely over $a_j$. If the second-term is non-zero, we have that $\vec{u}^\star$ is learned through the staircase mechanism, since it implies that $g^\star(\vec{h^\star})$ contains terms dependent on $h^*_1$  and linearly coupled with $\langle\vec{h^\star},\vec{u}^\star\rangle$. In either case, we obtain that $W^{(2)}$ almost surely obtains an overlap along $\vec{v}^*$. This concludes the proof of the first part of Theorem \ref{thm:main:2step_learning}.

More, generally, suppose that $\vec{e}_1,\cdots \vec{e}_m$ denote a basis of the directions in $U^*$ learned up to time $t$. Then, the modified condition for learned a new direction $\vec{v}^*$ at time $t+1$ is:
\begin{equation}\label{eq:f_hard}
        \Eb{\vec{z}}{F(f^\star(\vec{z}),z_1,\cdots z_m)\langle \vec{v}^\star, \vec{z}\rangle}\neq 0,
\end{equation}
for a polynomial $F:\R^{m+1} \rightarrow \R$. Therefore, new directions can be learned through a combination of the staircase and hidden-progress mechanism.

\subsection[Typical examples where 1]{Typical examples where $E_\perp^*=P^*_\perp = U^*$}
\label{sec:app:even_sym}

For several target functions of interest, the class $P^*_\perp$ can be shown to cover the entire target space $U^*$. We list some of them below:
\begin{itemize}
    \item Single-index odd polynomials with all non-negative/non-positive coefficients. This follows since $\Eb{\vec{z}}{(f^\star(\vec{z}))^k\langle \vec{v}^\star, \vec{z}\rangle}$ decomposes into sums of  non-negative/non-positive terms.
    \item Single-index odd Hermite polynomials. We prove this below in Lemma \ref{lem:odd_hermite}   
    \item Staircase function $f^*(\vec{z})=z_1+z_1z_2+z_1z_2z_3$. This follows directly by evaluating $\Eb{\vec{z}}{(f^\star(\vec{z}))^2z_i}$ for $i=2,3$.
\end{itemize}
 In general, for polynomial $f^\star$, the condition:
 \begin{equation}
     \Eb{\vec{z}}{(f^\star(\vec{z}))^k\langle \vec{v}^\star, \vec{z}\rangle} = 0, \forall k \in \N,
 \end{equation}
specifies an overdetermined system of infinite homogenous polynomial equations on the coefficients of $f^\star$. Therefore we expect the condition to fail almost surely for typical choices of $f^\star$. We leave an investigation of this using algebraic tools to future investigation.

\begin{lemma}\label{lem:odd_hermite}
For any odd Hermite-polynomial $H_{2k'+1}$ for $k' \in \N$,:
\begin{equation}
    \Eb{z}{(H_{2k'+1}(z))^3z}, \neq 0
\end{equation}
where $z \sim \mathcal{N}(0,1)$
\end{lemma}

\begin{proof}
Using Stein's Lemma, we have:
\begin{equation}\label{eq:stein}
    \Eb{z}{(H_{2k'+1}(z))^3z} = 3\Eb{z}{(H_{2k'+1}(z))^2  \frac{d}{dz}H_{2k'+1}(z)}
\end{equation}
Next, we recall the following relation between Hermite polynomials and their derivatives:
\begin{equation}
    \frac{d}{dz}H_{n}(z) = n H_{n-1}(z), \forall n \in \N.
\end{equation}
Substituting in Equation \eqref{eq:stein}, we obtain:
\begin{equation}\label{eq:stein2}
    \Eb{z}{(H_{2k'+1}(z))^3z} = 3(2k'+1)\Eb{z}{(H_{2k'+1}(z))^2 H_{2k'}(z)(z)}.
\end{equation}
The above expectation can be obtained analytically using the linearization formulas for Hermite polynomials \citep{andrews2004special} to show that $ \Eb{z}{(H_{2k'+1}(z))^2 H_{2k'}(z)(z)} \neq 0$ for all $k' \in \N$.
\end{proof}

\subsection{Proof of Proposition \ref{prop:OE}}\label{app:sec:proof:OE}

Suppose that $\vec{v}^\star \in OE^\star$ i.e $\vec{v}^\star$ is orthogonally even-symmetric w.r.t $f^\star$ for some transformation $O_{\perp} \in O(\{v^\star\}_\perp)$. Let $z' = O_{\perp}R_\vec{v^\star}\vec{z}$. Then, by the invariance of the Gaussian measure under orthogonal transformations, we have:
 \begin{equation}
   \Eb{\vec{z}}{f^\star(\vec{z})\langle \vec{v}^\star, \vec{z}\rangle} = \Eb{\vec{z}'}{f^\star(\vec{z}')\langle \vec{v}^\star, \vec{z}'\rangle}.
 \end{equation}
 However, the expectation on the right can equivalently be expressed as:
 \begin{align*}
     \Eb{\vec{z}'}{f^\star(\vec{z}')\langle \vec{v}^\star, \vec{z}'\rangle} &=-\Eb{\vec{z}}{f^\star(O_{\perp}R_{\vec{v}^\star}\vec{z}))\langle \vec{v}^\star, \vec{z}\rangle}\\
     &= -\Eb{\vec{z}}{f^\star(\vec{z})\langle \vec{v}^\star, \vec{z}\rangle}
 \end{align*}
 where in the second equality we used $\langle \vec{v}^\star, \vec{z}'\rangle = - \langle \vec{v}^\star, \vec{z}\rangle$ and in the third the definition \ref{def:ortho_symmetric_subspace}.
Therefore, for any $\vec{v}^\star \in OE^\star$, we have:
\begin{equation}
    \Eb{\vec{z}}{f^\star(\vec{z})\langle \vec{v}^\star, \vec{z}\rangle} = 0
\end{equation}

Furthermore, it is straightforward to see that  $\vec{v}^\star$ remains orthogonally even-symmetric w.r.t the composition $F(f^\star(\cdot))$. Therefore, we have that $E^\star \subseteq OE^\star \subseteq A^\star \subseteq P^\star$.

\subsection{Illustration of non-even symmetric hard directions} \label{sec:app:hard_sym}
\begin{figure} 
    \centering
    \includegraphics{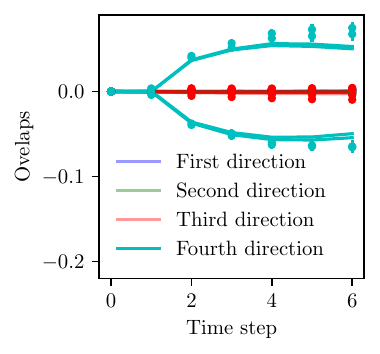} 
    \caption{
    An illustration of a hard, non-even target $f^\star(\vec z) \!=\! z_1 z_2 z_3 \!+\! \mathrm{He}_3(z_4)$ being learned by a student with $p=4$ hidden units. We can see that, \textit{even when reusing the batch}, the teacher can only learn the direction associated with $z_4$, while keeping a zero overlap otherwise. The continuous lines are from the DMFT numerical integration, the dots are simulations with $d\!=\!10000$. In the legend the overlap with the $n$-th direction is the projection of the student weights in the subspace associated with $z_n$. For this figure we have $\sigma \!=\! \rm relu$, $n=5d$, $\eta \!=\! 0.2$.}
    \label{fig:bad_example}
\end{figure}

We now show the existence of target functions where $E^* \neq OE^*$. Without loss of generality, we assume that the rows of $W^*$ lie along the standard Euclidean basis $\vec{e}_1, \vec{e}_2,\cdots, \vec{e}_k$
\begin{lemma}
    Suppose that $f^*(\vec{z}) = z_1z_2z_3$. Let $\vec{v}^* = \vec{e}_1+\vec{e}_2+\vec{e}_3$. Then $\vec{v}^* \notin E^*$ but $OE^*=U^*$.
\end{lemma}
\begin{proof}
    $\vec{v}^* \notin E^*$ follows directly by noting that $f^*(\vec{z})$ is even-symmetric along $\vec{e}_1-\vec{e}_2$ and $\vec{e}_1-\vec{e}_3$. We further have that a target $f^*$ satisfying $E^*=U^*$ must satisfy $f^*(-\vec{z}) = f^*(\vec{z}) \ \forall \vec{z} \in \R^d$. Therefore, since $f^*(-\vec{z}) = -f^*(\vec{z})$, $f^*$ cannot be even-symmetric along $\vec{e}_1+\vec{e}_2+\vec{e}_3$. 
    Next, we show that $OE^*=U^*$. Since $\vec{e}_1, \vec{e}_2, \vec{e}_3$ span $U^*$, and $f^*$ is symmetric w.r.t permutations of $z_1z_2z_3$, it suffices to show that condition in definition \ref{def:ortho_symmetric_subspace} holds for $\vec{v}^* = \vec{e}_1$. The orthogonal complement $\{\vec{v}^*\}_\perp$ is given by $\operatorname{span}(\vec{e}_2,\vec{e}_3)$. Therefore the transformation $O_2$ defined by $z_2 \rightarrow -z_2$ is a valid orthogonal transformation $O_\perp$ as per definition \ref{def:ortho_symmetric_subspace}. We have:

\begin{align*}
    f^*(O_2 R_{\vec{v}^*}\vec{z}) &= (-z_1)(-z_2)z_3\\
    &= z_1z_2z_3 = f^*(\vec{z}).
\end{align*}
This shows that $\vec{e}_1$ lies in $OE^*$. Similarly, we have by symmetry $\vec{e}_1 \in OE^*$

We present a numerical illustration of another such example in figure \ref{fig:bad_example}.
\end{proof}

One can in-fact construct a family of functions with a direction $\vec{v}^*$, for instance $\vec{v}^* = \vec{e_1}$ lying in $OE^*$ but in general not in $E^*$. To see this, let $f_1$ be a function $\R^d \rightarrow \R$, depending only on projections of $\vec{z}$ along $\{\vec{e}_1\}_\perp$ and let $O_\perp$ by an involutory orthogonal transformation on $\{\vec{e}_1\}_\perp$ i.e an orthogonal transformation satisfying $O^2_\perp = \mathbf{I}$ or equivalently $O_\perp = (O_\perp)^\top$. Now, let $f_2:\R \rightarrow \R$ be an odd function. Then, consider the function:
\begin{equation}
    f^*(\vec{z}) = (f_1(O_\perp \vec{z})-f_(\vec{z}))f_2(z_1).
\end{equation}

We observe that:
\begin{align*}
    f^*(O R_{\vec{e}_1}\vec{z}) &= (f_1(O^2_\perp\vec{z})-f_(O_\perp\vec{z}))f_2(-z_1)\\
    &= (f_1(O_\perp\vec{z})-f_(\vec{z}))f_2(z_1)\\
    &= f^*(\vec{z}),
\end{align*}
where we used that $O^2_\perp = \mathbf{I}$ and $f_2(-z_1)=-f_2(z_1)$. Therefore, for any such function $f^*(\vec{z})$, $\vec{e}_1 \in OE^*$. 



\subsection{Implications for generalization} \label{sec:app:gen_error}
Since the specific guarantees of such results depend on the choice of activation and target functions, we illustrate this for the case of single-index target functions with matching activations:
\begin{corollary}
Consider the setting of a single-index target and student network with matching activations i.e. $\sigma=g^\star$, such that $\sigma$ is a polynomial with finite degree, satisfying the following assumption, $\exists k \in \N$ such that:
\begin{equation}\label{eq:matching}
    \Ea{\sigma^k(z)z}\Ea{D^k\sigma(z)} \neq 0, 
\end{equation},
where $D^k$ denotes the $k_{th}$ derivative.
Let $\hat{\vec{w}}$ be the parameters obtained after two steps of gradient descent with batch size $\cO(d)$ using $\eta$ as in Theorem \ref{thm:main:2step_learning}. Then, almost surely over the initialization $a \sim \mathcal{N}(0,1)$, for any $\epsilon > 0$, there exists a step size $\eta'$ such that online SGD on 
squared loss reaches generalization error $<epsilon$ in time $\cO(d)$. 
\end{corollary}

We verify numerically that the above assumption holds in particular for all odd Hermite polynomials upto order $50$.
The corollary implies that such target functions can be learned with $\cO(d)$ sample complexity using gradient descent alone, without resorting to specialized algorithms and techniques such as spectral initialization.

\begin{proof}
    Let $\vec{w}^*$ denote the single-direction in the teacher subspace with $\norm{\vec{w}^*}=\sqrt{d}$.
    We note that Equation \eqref{eq:matching} is proportional to the ${k-1}_{th}$ derivative of $\phi(a_j)$ defined in Equation \eqref{eq:phi_def}. Therefore, the condition is sufficient to ensure that $\phi(a_j)$ is not identically zero and the student neuron almost surely develops an overlap along $\vec{w}^*$.
    The result then follows from
    Proposition 2.1 in \citep{BenArous2021}, which proves that upon weak recovery i.e a non-zero overlap the target direction $\vec{v}^*$, online SGD on a differentiable activation with polynomially bounded derivatives converges to strong recovery Concretely, for any starting non-zero overlap $\theta > 0$, for any $\epsilon' > 0$, there exists $C_{\epsilon',\theta}$ and small-enough step-size such that online SGD with time $C_{\epsilon',\theta} d$ achieves overlap $1-\epsilon'$ along $\vec{w}^*$
    Due to the matching activations, this suffices to obtain arbitrary generalization error.
    \end{proof}

\section{General Multi-Pass Schemes}\label{sec:app:repeat}

\begin{figure*}[t]
    \centering
    \includegraphics{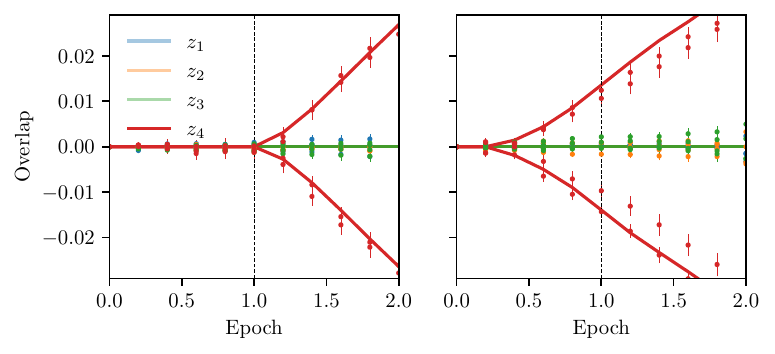}
    \caption{Comparison of theory and experiments for Gradient Descent on the target $z_1z_2z_3 + \mathrm{He}_3(z_4)$. Each gradient step uses a mini-batche of $n/5$ samples. On the {\bf left} we use the data sequentially, on the {\bf right} we sample the batch from the dataset with replacement. The continuous lines are from the DMFT numerical integration, the dots are simulations with $d\!=\!10000$ averaged over $32$ realisations. In the legend the overlap with the $n$-th direction is the projection of the student weights in the subspace associated with $z_n$. For this figure we have $\sigma \!=\! \rm relu$, $n=5d$, $\eta \!=\! 0.2$.}
    \label{fig:overlaps_batch}
\end{figure*}

\subsection{Sketch of Proof for Extending Theorem \ref{thm:main:2step_learning} to Cycling over Epochs}
Let $\vec{Z}^1,\cdots, \vec{Z}^n_e$ denote $n_e$ independent minibatches of size $n_b$ such that $\frac{n_b}{d} = \mathcal{O}(1)$ with $n_e$ being finite. The effective dynamics for a finite number of epochs can be obtained by noting that Theorem 3.2 in \cite{gerbelot2023rigorous} allows generalizing Theorem \ref{thm:Cedric_theorem} to dynamics of the form:
 \begin{align}
        W^{(t+1)} =  W^{(t)} - \eta \lambda W^{(t)} -  \sum_{i=1}^{n_e}\eta \frac{1}{\sqrt{d}}\sum_{\nu=1}^{n_b} F^t_i\left(\frac{ W^{(t)}\vec z^i_\nu} {\sqrt{d}}\right) (\vec z^i_\nu)^\top.
    \end{align}
The above form of the dynamics allows a different update to be utilized for data corresponding to different blocks  $\vec{Z}^1,\cdots, \vec{Z}^n_e$. In particular, setting $F^t_i$ to $0$ whenever $t\mod i \neq 0$  and $\nabla_{\boldsymbol{h}} \mathcal{L}$ otherwise, results in a cycling schedule over the mini-batches $\vec{Z}^1,\cdots, \vec{Z}^n_e$. Subsequently, one can show that the update from $\vec{Z}^i$ in the first-epoch leads to the hidden-progress effect on $M^{t}$ when the model re-uses $\vec{Z}^i$ in the second epoch.

We believe a similar result would hold for $n_b = \mathcal{O}(1)$ samples in the minibatch, as displayed in Figure \ref{fig:overlaps_minibatch}

\begin{figure}[t]
    \centering
    \includegraphics{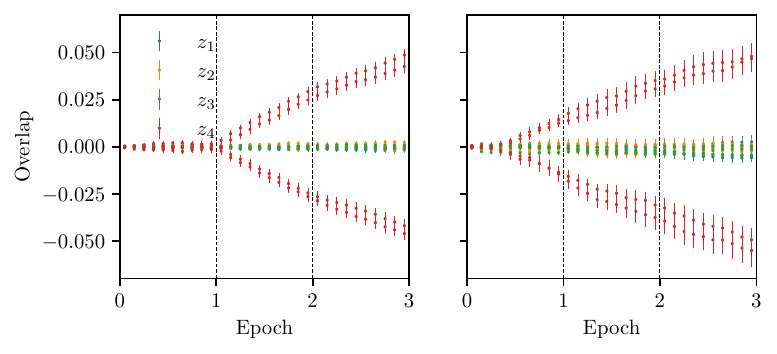}
    \caption{Experiments for Gradient Descent on the target $z_1z_2z_3 + \mathrm{He}_3(z_4)$. We use minibatches with $1$ sample each. On the {\bf left} we use the data sequentially, on the {\bf right} we sample the data point from the dataset with replacement. The dots are simulations with $d\!=\!10000$ averaged over $32$ realisations. In the legend the overlap with the $n$-th direction denotes the projection of the student weights in the subspace associated with $z_n$. For this figure we have $\sigma \!=\! \rm relu$, $n=5d$, $\eta \!=\! 0.2$.}
    \label{fig:overlaps_minibatch}
\end{figure}

%% file: sections/appendix/details_numerics.tex
\section{Details on the numerics}
\label{sec:app:hyperparams}

\subsection{DMFT equations with a single stochastic process}
In this section, we present a set of exact equations equivalent to the ones in the main text, but that depend on a single stochastic process.
 It is possible to show that asymptotically in the proportional limit, i.e. for $d\to\infty$ and $n=\alpha d$, the pre-activations of the student are distributed as $\vec h^{(t)} = \vec r^{(t)} + M^{(t)}\vec h^*$, with the constraint:
\begin{align} \label{eq:r_update}
    \vec r^{(t+1)} = \vec r^{(t)} -\eta \left[ \left(\lambda+\Lambda^{(t)}\right) \vec r^{(t)} + \nabla_{\vec h^{(t)}}\ell\left( \vec h^{(t)}\right)-\sum_{\tau=0}^{t-1} R_\ell^{(t,\tau)}\vec r^{(t)}+\vec\zeta^{(t)}\right]\nonumber
\end{align}
Here $\vec\zeta^{(t)}$ is a zero mean Gaussian Process with covariance
\begin{align}
    \mathbb{E}_{\vec \zeta}\left[\vec\zeta^{(t)} \vec\zeta^{(\tau)\top}\right] = \alpha\mathbb{E}_{\vec r, \vec h^*}\left[\nabla_{\vec h^{(t)}}\ell\left( \vec h^{(t)}\right)\nabla_{\vec h^{(\tau)}}\ell\left( \vec h^{(\tau)}\right)^\top\right]\nonumber
\end{align}
and the effective regularisation $\Lambda^{(t)}$ concentrates to
\begin{align} 
    \Lambda^{(t)} = \alpha\mathbb{E}_{\vec h}\left[\nabla^2_{\vec h^{(t)}}\ell\left( \vec h^{(t)}\right)\right]\,.
\end{align}
The memory kernel $R_\ell^{(t,\tau)}$ is identically zero for $t\leq \tau$ while for $t > \tau$ it concentrates to
\begin{align} 
    R_\ell^{(t,\tau)} = \alpha \mathbb{E}_{\vec h}\left[\frac{\partial\, \nabla_{\vec h^{(t)}}\ell\left( \vec h^{(t)}\right)}{\partial\, \zeta^{(\tau)}} \right]
\end{align}
Finally, the low dimensionaly projections of the weights $M^{(t)}$ will obey the relation
\begin{align} 
    M^{(t+1)} = M^{(t)} - \eta \alpha \mathbb{E}_{\vec h, \vec h^*}\left[\nabla_{\vec h^{(t)}}\ell\left( \vec h^{(t)}\right) \vec h^{*\top}\right]
\end{align}
The procedure is explained in detail in appendix D of \citep{gerbelot2023rigorous}, and can be equivalently derived using non-rigorous field theory techniques \citep{ABUZ18}.

\subsection{Remark on the numerical integration of the DMFT equations}
DMFT is an invaluable tool in itself to probe the behaviour of gradient based algorithms. It trades the update equation over heavily coupled weights in \eqref{eq:GD_update} with the ones over completely decoupled preactivations \eqref{eq:r_update} which implies that a Monte Carlo estimation based on \eqref{eq:r_update} is going to be vastly more efficient and it's a trivially parallelisable computation.
Furthermore, equation \eqref{eq:r_update} is exact in limit of large $d$, which removes completely all finite size effects.
In practice, an implementation of the DMFT equations is extracting $n$ times using from the initial condition distribution of the practivations and iterating forward.
The Gaussian process is sampled by rotating white Gaussian noise by the LU factor of the covariance. Sampling the gaussian process is by far the costlier operation, as each time step $T$ has a complexity $\mathcal{O}(T^3)$, for a total $\mathcal{O}(T^4)$ complexity considering all the steps up to $T$.
Notice that this is a much more direct implementation than what is done in the literature \citep{Roy_2019, mignacco2020dynamical}, which usually starts with a guess for all the quantities and proceedes with a damped fixed point iteration until convergence, with an overall complexity $\mathcal{O}(mT^3)$, where $m$ is the number of fixed point iterations. While it could appear that simply iterating forward is suboptimal, it is a much more stable and reliable procedure: if you are using $n$ processes and you iterate forward, you are sure that at at each time step you have the best possible Monte Carlo estimate of your samples. 

\subsection{Details on the numerical simulations}
In all the figures the continuous lines are from the numerical integration of the DMFT equations while the dots are from a direct simulation of the gradient descent dynamics. The specific hyperparameters for each setting are near each figure.

For both we fixed the second layer weights to $\pm 1/\sqrt{p}$, as for the cases under consideration this is an equivalent choice to of Gaussian second layer weights $\mathcal{N}(0,\frac{1}{p}{\mathbbm 1}_p)$. For the DMFT integration we used a minimum of $10^6$ Monte Carlo samples in order to have accurate lined. The error bars are too small to be visualised. 
The direct simulation of the gradient descent dynamics was performed either using PyTorch or a direct implementation in Numpy. In all plots we used a minimum size $d=5000$ for the input dimension, and averaged over at least $32$ independent instances of the dynamics.

In Figure \ref{fig:main:multiple_index} we plot the overlap matrix $M^{(t)}$ projected on two different directions: the parallel to the subspace that is learned in the first step and one direction in the orthogonal of this space.
The projection operator is computed by performing explicitly the integrals in \eqref{eq:M_fist_step}

The code is made available through the following Github repository: \href{https://github.com/IdePHICS/benefit-reusing-batch}{https://github.com/IdePHICS/benefit-reusing-batch}.